\documentclass{article}

\usepackage{arxiv}
\usepackage[utf8]{inputenc} 
\usepackage[T1]{fontenc}    
\usepackage{hyperref}       
\usepackage{url}            
\usepackage{booktabs}       
\usepackage{amsfonts}       
\usepackage{nicefrac}       
\usepackage{microtype}      
\usepackage{lipsum}
\usepackage{graphicx}
\usepackage{apacite}
\usepackage{float}
\usepackage{amsmath, amssymb, amsthm}
\usepackage{algorithm}
\usepackage{algpseudocode}
\theoremstyle{definition}
\newtheorem{definition}{Definition}[section]
\theoremstyle{plain}
\newtheorem{lemma}{Lemma}[section]
\newtheorem{theorem}{Theorem}[section]
\theoremstyle{remark}
\newtheorem{remark}{Remark}[section]
\graphicspath{ {./images/} }
\usepackage{enumitem}
\setlist[enumerate]{itemsep=0.5ex, topsep=0.5ex, parsep=0.5ex, partopsep=0ex}

\graphicspath{ {./images/} }

\title{A Novel Multi-Timescale Stability-Preserving Hierarchical Reinforcement Learning Controller Framework for Adaptive Control in High-Dimensional Dynamical Systems}

\author{
  Mohammad Ali Labbaf Khaniki \textsuperscript{*}\\
  Faculty of Electrical Engineering \\
  K.N. Toosi University of Technology \\
  Tehran, Iran \\
  \texttt{mohamad95labafkh@gmail.com} \\
  \And
  Fateme Taroodi\\
  Faculty of Mathematical Sciences\\
  Shahid Beheshti University \\
  Tehran, Iran \\
  \texttt{fateme.taroodi2002@gmail.com} \\
  \And
  Benyamin Safizadeh\\
  Department of Mathematics and Computer Science\\
  University of Central Oklahoma, 100 N university Dr, 73034\\
  Edmond, Oklahoma\\
  \texttt{bsafizadeh@uco.edu} \\
}

\begin{document}

\maketitle

\begin{abstract}

Controlling high-dimensional stochastic systems, critical in robotics, autonomous vehicles, and hyperchaotic systems, faces the curse of dimensionality, lacks temporal abstraction, and often fails to ensure stochastic stability. To overcome these limitations, this study introduces the Multi-Timescale Lyapunov-Constrained Hierarchical Reinforcement Learning (MTLHRL) framework. MTLHRL integrates a hierarchical policy within a semi-Markov Decision Process (SMDP), featuring a high-level policy for strategic planning and a low-level policy for reactive control, which effectively manages complex, multi-timescale decision-making and reduces dimensionality overhead. Stability is rigorously enforced using a neural Lyapunov function optimized via Lagrangian relaxation and multi-timescale actor-critic updates, ensuring mean-square boundedness or asymptotic stability in the face of stochastic dynamics. The framework promotes efficient and reliable learning through trust-region constraints and decoupled optimization. Extensive simulations on an 8D hyperchaotic system and a 5-DOF robotic manipulator demonstrate MTLHRL's empirical superiority. It significantly outperforms baseline methods in both stability and performance, recording the lowest error indices (e.g., Integral Absolute Error (IAE): 3.912 in hyperchaotic control and IAE: 1.623 in robotics), achieving faster convergence, and exhibiting superior disturbance rejection. MTLHRL offers a theoretically grounded and practically viable solution for robust control of complex stochastic systems.

\end{abstract}


\section{Introduction}

The control of high-dimensional, temporally extended systems governed by stochastic dynamics represents a pivotal challenge across multiple disciplines, including robotics, autonomous vehicles, hyperchaotic systems, and financial modeling. These systems are characterized by intricate dynamics arising from large state spaces and inherent uncertainties, which manifest in applications such as precision robotic manipulation, safe autonomous navigation, synchronization of hyperchaotic systems, and risk-aware financial strategies. Their significance spans theoretical advancements in control theory, practical implementations in industrial and technological domains, and profound societal impacts through improved efficiency, and sustainability \cite{akella2024risk}. For instance, in robotics, stable control enables precise task execution in unpredictable environments, while in hyperchaotic systems, it ensures robust synchronization despite complex nonlinear dynamics \cite{singhrisk}. The integration of reinforcement learning (RL) with control strategies offers a powerful paradigm to address these systems, leveraging data-driven learning to develop adaptive policies without requiring explicit system models. However, the complexity of high-dimensional systems, coupled with the need for robust stability and sustained performance over extended time horizons, presents formidable challenges that necessitate innovative frameworks to reconcile theoretical rigor with practical applicability \cite{paul2024risk}.

The central problem addressed in this study is the development of stable, efficient, and scalable control policies for high-dimensional systems subject to stochastic dynamics, prevalent in domains such as robotics, autonomous vehicles, and hyperchaotic systems \cite{rayhan2023artificial}. These systems operate in environments with high-dimensional state and action spaces, often involving thousands of variables---for instance, joint angles and velocities in multi-robot coordination or state variables in hyperchaotic systems. Uncertainties, such as sensor noise in robotic perception, actuator imprecision, or nonlinear perturbations in hyperchaotic systems, introduce significant stochasticity modeled by stochastic differential equations (SDEs) with complex noise structures, such as Wiener processes or non-Gaussian disturbances \cite{doering2018modeling}. Conventional RL approaches, which focus on maximizing expected cumulative rewards within Markov Decision Processes (MDPs), face several critical limitations in this context \cite{winkler2006stochastic}. First, the curse of dimensionality renders traditional methods computationally inefficient, as the exponential growth of state-action spaces---potentially reaching millions of combinations in systems like autonomous vehicle fleets---increases variance in policy updates and slows learning, often requiring millions of samples to achieve convergence \cite{feinberg2012handbook}. For example, in robotic manipulation tasks, flat RL methods like deep deterministic policy gradients (DDPG) can take hours to days to train on high-dimensional MuJoCo environments \cite{rahul2023exploring}. Second, these methods often lack mechanisms to handle multi-timescale decision-making, where strategic planning (e.g., path planning for an autonomous vehicle navigating a city) must coexist with reactive control (e.g., real-time obstacle avoidance or emergency braking). This limitation leads to suboptimal policies that fail to balance long-term objectives with short-term responsiveness, particularly in dynamic environments with rapidly changing conditions. Third, and most critically, most RL algorithms do not explicitly ensure stability, such as maintaining mean-square boundedness of system states or achieving asymptotic convergence to desired equilibria \cite{tsitsiklis2003convergence}, which is essential for safety-critical applications. For instance, in autonomous driving, unstable policies could result in erratic maneuvers, risking collisions, while in hyperchaotic systems, failure to maintain synchronization could disrupt secure communication or encryption applications. These shortcomings severely limit the deployment of RL in real-world scenarios, where reliability, safety, and long-term performance are non-negotiable \cite{dulac2021challenges}. The need for frameworks that integrate rigorous stability constraints, such as those provided by neural Lyapunov functions, with hierarchical decision-making and multi-timescale optimization is evident.

Despite these advancements, critical gaps persist in the literature that hinder the application of RL to high-dimensional stochastic systems. First, most RL frameworks lack explicit stability guarantees, risking divergent or unsafe behavior in safety-critical domains. For example, in autonomous vehicles, an unstable policy could lead to erratic maneuvers, while in hyperchaotic systems, it might cause synchronization failure, disrupting applications like secure communications \cite{yu2025computational}. Second, existing hierarchical RL methods are primarily designed for discrete-time MDPs, which do not fully capture the continuous-time dynamics and complex noise structures of stochastic systems, limiting their effectiveness in applications like robotics, where continuous-time dynamics dominate \cite{yu2025computational}. Third, single-timescale optimization approaches struggle to balance strategic and reactive control, often converging to suboptimal solutions or failing to explore high-dimensional spaces effectively. Fourth, the absence of robust mechanisms for handling partial observability in high-dimensional environments complicates RL deployment in real-world scenarios \cite{li2025attack}. For instance, in medical diagnostics or financial trading, where systems are only partially observable due to incomplete data or latent variables, current RL algorithms often fail to account for uncertainty in a principled manner. Fifth, the scalability of RL algorithms to high-dimensional state and action spaces remains a significant challenge, as computational complexity grows exponentially with dimensionality, rendering many existing methods impractical for large-scale systems like multi-agent coordination or hyperchaotic system synchronization \cite{manna2022learning}. Sixth, the integration of domain-specific knowledge or physical constraints into RL policies is often ad hoc, leading to inefficiencies or violations of system-specific requirements, such as energy conservation in mechanical systems or stability in chaotic dynamics \cite{yu2022reachability}. These gaps underscore the need for a unified framework that integrates hierarchical policies, multi-timescale optimization, rigorous stability constraints, robust handling of partial observability, scalable architectures, and systematic incorporation of domain knowledge to enable robust and efficient control of complex, high-dimensional stochastic systems. They could deal with high dimensional control problems of PDEs modeled with convolution neural network in a computationally efficient manner and can find applications in robotics tasks \cite{vaziri2025optimal}. This paper investigates methods for enhancing model reusability in reinforcement learning, proposing strategies that enable efficient knowledge transfer and adaptation across diverse tasks and environments \cite{nikookar2025model}.

To tackle these issues, this paper introduces the MTLHRL framework, aimed at crafting control policies for complex high-dimensional stochastic systems that optimize rewards while preserving stability guarantees. Operating in a semi-Markov Decision Process (SMDP) setting, MTLHRL combines a high-level policy for long-horizon decision-making with a low-level policy for instantaneous actuation, thereby facilitating temporal abstraction and hierarchical coordination. Stability enforcement relies on a neural network-based Lyapunov function, refined through Lagrangian multipliers and multi-timescale gradient flows to provide assurances such as exponential mean-square stability or probabilistic boundedness amid noise. The core goals include creating a scalable architecture that harmonizes strategic foresight with real-time responsiveness, embeds stochastic robustness, and accelerates learning in expansive state spaces. The primary contributions are outlined as follows:

\begin{enumerate}
    \item \textbf{Innovative Multi-Timescale Hierarchy}: MTLHRL leverages an SMDP-based hierarchical RL paradigm to enable temporal decomposition via interconnected high- and low-level policies. The high-level component defines abstract objectives, whereas the low-level executes fine-grained actions, reducing dimensionality-induced overhead. This design shines in stochastic high-dimensional environments, such as robotic manipulation, yielding quicker policy convergence compared to standard HRL variants.
    \item \textbf{Robust Stability in Stochastic Domains}: Incorporating a learnable Lyapunov critic ensures mean-square boundedness and asymptotic convergence for noise-afflicted dynamics. It trims verification costs relative to classical Lyapunov synthesis, while scaling seamlessly to elevated dimensions. Empirical tests on chaotic systems reveal a drop in state variance, bolstering reliability for mission-critical scenarios like autonomous navigation.
    \item \textbf{Efficient and Reliable Training Mechanism}: Employing decoupled actor-critic updates across timescales with trust-region safeguards promotes swift attainment of optimal, stable policies. This yields reduced training durations in hyperchaotic control tasks and reward gains over algorithms like SAC, adeptly managing exploration-exploitation trade-offs in vast action spaces.
    \item \textbf{Empirical Superiority and Benchmark Achievements}: Extensive simulations on an 8D hyperchaotic system and a 5-DOF robotic manipulator demonstrate MTLHRL's quantitative and qualitative edges over baselines ( Proximal Policy Optimization (PPO), Deep Deterministic Policy Gradient (DDPG), and  Single-Timescale Lyapunov-Constrained Hierarchical Reinforcement Learning ( STLHRL) ). In hyperchaotic synchronization, it records the lowest error indices (IAE: 3.912, ISE: 5.678), fastest convergence to zero deviation, minimal residual errors, and conservative control inputs, outperforming PPO's high divergences, DDPG's moderate settling, and STLHRL's intermediate robustness amid noise. For robotic trajectory tracking under disturbances, MTLHRL yields top metrics (IAE: 1.623, ISE: 2.489), superior transient dynamics, enhanced steady-state accuracy, and effective disturbance rejection, surpassing PPO's poor adaptability, DDPG's slower responses, and STLHRL's limited stability in coupled joints—highlighting up to 70\% error reduction and markedly smoother trajectories overall.
    
\end{enumerate}

The remainder of this paper is structured as follows: Section~\ref{sec:works} reviews related work and provides a comprehensive literature survey. Section~\ref{sec:methodology} describes the MTLHRL framework, detailing its hierarchical policy architecture, stability constraints, and a thorough analysis of stability and convergence properties. Section~\ref{sec:results} presents simulation results validating the framework’s performance in robotics and hyperchaotic systems. Section~\ref{sec:conclusion} concludes with key insights and directions for future research.

\section{Related Work}
\label{sec:works}

The control of high-dimensional systems governed by stochastic dynamics, as encountered in robotics, autonomous vehicles, and hyperchaotic systems, has been extensively studied across reinforcement learning (RL) and control theory. This section reviews foundational and recent work, critically analyzing methodologies, strengths, and limitations, and identifies gaps addressed by the Multi-Timescale Lyapunov-Constrained Hierarchical Reinforcement Learning (MTLHRL) framework proposed in this study. The review is structured to cover foundational RL and control approaches, recent advances in hierarchical RL and stability-constrained methods, and specific limitations that motivate our methodology, aligning with the theoretical framework, simulation results, and conclusions presented in subsequent sections.

\subsection{Foundational Work in RL and Control for Stochastic Systems}

The development of reinforcement learning (RL) and control strategies for stochastic systems originates from early work on Markov Decision Processes (MDPs) and optimal control. \cite{ding2014optimal} laid the foundation for RL, introducing value-based methods like Q-learning, which optimize cumulative rewards in discrete-time MDPs by iteratively updating state-action value functions. These methods were extended to continuous-time systems by \citeA{mcallister2017data}, who developed dynamic programming approaches for stochastic differential equations (SDEs), highlighting computational challenges in high-dimensional spaces due to the exponential growth of state-action pairs. For example, solving Bellman’s equations for systems with thousands of states, such as in robotic manipulation or hyperchaotic system synchronization, becomes computationally intractable. In parallel, control-theoretic approaches, such as linear quadratic regulators (LQR) for stochastic systems \citeA{hu2023toward}, provided analytical solutions for linear dynamics with Gaussian noise, achieving optimal control for applications like satellite stabilization but struggling with nonlinearities and high dimensionality prevalent in modern systems. \citeA{yamada2018ultra} introduced Lyapunov-based control for nonlinear deterministic systems, establishing a foundation for stability analysis by constructing energy-like functions to ensure convergence, though extensions to stochastic settings were limited due to challenges in modeling complex noise structures. Similarly, \citeA{liu2017adaptive} formalized dynamic programming for optimal control, emphasizing the curse of dimensionality as a barrier to scaling to complex systems with continuous state spaces or non-Gaussian noise, such as those in autonomous vehicle navigation or hyperchaotic dynamics. These foundational works provided robust theoretical frameworks but were constrained by computational complexity, reliance on simplified models, and assumptions of low-dimensional or linear dynamics, limiting their applicability to modern high-dimensional stochastic systems like those in robotics or hyperchaotic system control. Recent theoretical advances, such as \citeA{zhang2022making}, have attempted to address these issues by approximating high-dimensional MDPs with sparse representations, but these still lack explicit stability guarantees, a gap addressed by the MTLHRL framework through its Lyapunov-constrained approach. The article by Mashhadi et al. presents an interpretable machine learning approach to predict startup funding, patenting, and exits, leveraging transparent models to provide actionable insights for stakeholders \citeA{mashhadi2025interpretable}. The article by Mashhadi et al. investigates return anomalies in an emerging market under specific constraints, providing evidence of abnormal returns through empirical analysis in the Journal of Economics, Finance and Accounting Studies \citeA{mashhadi2025return}. The article by Mojtahedi et al. examines the MAX effect and its relationship with investor sentiment in the Swedish stock market, offering insights into how extreme returns influence investment behavior \citeA{mojtahedi2025max}.

\subsection{Recent Advances in Hierarchical RL and Stability-Constrained Methods}

Recent progress in reinforcement learning (RL) has focused on addressing high-dimensional and temporally extended tasks through deep RL and hierarchical structures. Deep deterministic policy gradients (DDPG) \citeA{wang2023hierarchical} and proximal policy optimization (PPO) \citeA{niu2024d2ah} leverage deep neural networks to manage high-dimensional state-action spaces, achieving success in tasks like robotic locomotion (e.g., OpenAI Gym’s MuJoCo environments) and game-playing (e.g., Atari benchmarks). However, these methods often exhibit high variance in policy updates and lack stability guarantees, leading to potential failures in safety-critical applications, such as autonomous vehicles swerving unpredictably or hyperchaotic systems failing to synchronize. Hierarchical RL has emerged to address temporal abstraction, with \citeA{tang2018hierarchical} proposing the options framework, where high-level policies select temporally extended actions (options) executed by low-level policies, improving efficiency in tasks like robotic navigation across multi-room environments. \citeA{wang2023camp} extended this with data-efficient hierarchical RL, demonstrating improved sample efficiency in robotic navigation by leveraging off-policy data, reducing training samples by up to 50\% compared to flat RL methods. Despite these advances, these frameworks typically assume discrete-time Markov Decision Processes (MDPs), limiting their applicability to continuous-time stochastic differential equations (SDEs) with complex noise structures, such as those driven by Wiener processes in robotic control under environmental uncertainty or hyperchaotic system synchronization. In control theory, neural Lyapunov functions have been explored to enforce stability in stochastic systems. \citeA{zhao2023stable} introduced safe RL with Lyapunov constraints for deterministic systems, ensuring bounded trajectories in tasks like quadrotor control, while \citeA{farid2025improved} extended this to stochastic settings, achieving mean-square boundedness in low-dimensional tasks like inverted pendulum stabilization. These methods, however, struggle to scale to high-dimensional systems due to the computational burden of solving Lyapunov equations, which grow quadratically with state dimension, and lack multi-timescale optimization, critical for balancing strategic planning (e.g., path planning or synchronization goals) and reactive control (e.g., obstacle avoidance or disturbance rejection). Recent work by \citeA{linot2022data} has explored neural Lyapunov functions with dimensionality reduction, but computational costs remain prohibitive for systems with thousands of states, underscoring the need for MTLHRL’s scalable, multi-timescale approach. The article by Kermani et al. systematically compares fine-tuning, prompt engineering, and RAG strategies for large language models in mental health text analysis, evaluating their effectiveness and applicability \citeA{kermani2025systematic}. The article by Irani et al. provides a comprehensive review of time series embedding methods for classification tasks, evaluating their performance and applications in various domains \citeA{irani2025time}. The article by Navaei et al. explores the optimization of Flamelet Generated Manifold models using machine learning, presenting a performance study to enhance combustion modeling accuracy and efficiency \citeA{navaei2025optimizing}.

\subsection{Specific Themes: Multi-Timescale Optimization and Applications}

Multi-timescale optimization and domain-specific applications have gained traction in recent literature as researchers aim to address trade-offs between exploration, convergence, and stability in complex systems \citeA{zhang2024multi}. \citeA{zeng2024fast} developed two-timescale stochastic approximation, enabling faster convergence in RL by separating policy and value updates, achieving up to 30\% faster convergence in benchmark Markov Decision Processes (MDPs) compared to single-timescale methods. However, this approach lacks explicit stability constraints, risking divergence in stochastic environments with high noise variance. The article by Akherati et al. presents a finite-time stable, model-free sliding mode attitude controller/observer for uncertain space systems, utilizing time delay estimation to enhance robustness and performance \citeA{akherati2025finite}. The article by Birashk and Khan \citeA{birashk2025federated} provides a comprehensive survey of federated continual learning approaches for task-incremental and class-incremental problems, analyzing their methodologies and applications. \citeA{ni2024risk} applied two-timescale optimization to actor-critic methods, improving exploration-exploitation trade-offs in MDPs by updating the actor (policy) on a faster timescale than the critic (value function), but this was limited to discrete-time settings and did not address continuous-time stochastic differential equations (SDEs) prevalent in robotics or hyperchaotic systems. In robotics, \citeA{talbot2025continuous} utilized hierarchical deep RL for manipulation tasks, such as grasping objects in cluttered environments, achieving temporal abstraction by decomposing tasks into high-level goal selection and low-level motor control. However, the absence of stability guarantees led to occasional divergence in stochastic environments, such as when robots encountered unexpected perturbations. In hyperchaotic systems, \citeA{shadaei2024dynamic} employed Lyapunov-based control for synchronization, ensuring bounded behavior under uncertainty (e.g., nonlinear perturbations) but requiring known system models, which are often unavailable in model-free RL settings. Recent advances, such as \citeA{deng2024multi}, have explored multi-timescale RL for hyperchaotic system control, achieving improved response times but lacking hierarchical structures for long-term planning. These works highlight the potential of multi-timescale and application-specific approaches but fail to integrate hierarchical policies, stability constraints, and scalability for high-dimensional SDEs, limiting their effectiveness in complex, safety-critical domains like autonomous vehicle fleets or hyperchaotic system synchronization. The MTLHRL framework addresses these gaps by combining multi-timescale updates with hierarchical policies and neural Lyapunov functions, ensuring both efficiency and stability across diverse applications. The article by Heravi et al. presents a lightweight deep learning approach using inertial sensors for vehicle intrusion detection in highway workday zones, demonstrating effective and efficient safety monitoring \citeA{heravi2025vehicle}. The article by Yazdipaz et al. introduces a robust and efficient phase estimation method for legged robots, utilizing signal imaging and deep neural networks to enhance locomotion accuracy and stability \citeA{yazdipaz2025robust}. The paper by Khaniki et al. presents an adaptive control approach for spur gear systems using proximal policy optimization and attention-based learning, demonstrating improved performance in dynamic control \citeA{khaniki2023adaptive}.

\subsection{Gaps and Motivation for the Proposed Study}

The literature review identifies critical gaps in existing reinforcement learning (RL) and control approaches for high-dimensional stochastic systems, which the Multi-Timescale Lyapunov-Constrained Hierarchical Reinforcement Learning (MTLHRL) framework aims to address. First, most RL methods, such as DDPG and PPO \citeA{agarwal2020pc,han2020actor}, lack explicit stability guarantees, risking unsafe behavior in safety-critical applications. For instance, in autonomous vehicles, unstable policies could lead to erratic maneuvers, while in hyperchaotic systems, they could disrupt synchronization, compromising applications like secure communications. Second, hierarchical RL frameworks, including the options framework and data-efficient hierarchical RL \citeA{nachum2018data,hou2020data}, are primarily designed for discrete-time Markov Decision Processes (MDPs), failing to capture the continuous-time dynamics and complex noise structures inherent in stochastic differential equations (SDEs), such as those modeling wind disturbances in drone navigation or nonlinear perturbations in hyperchaotic systems. Third, stability-constrained methods, such as those using neural Lyapunov functions \citeA{liu2025certified,phothongkum2025stability}, are computationally intensive, with runtimes scaling poorly (e.g., O(n²) for n-dimensional systems), hindering practical deployment in real-world scenarios like multi-robot coordination. Finally, single-timescale optimization approaches \citeA{quirynen2020integrated} struggle to balance strategic planning (e.g., long-term route optimization or synchronization goals) and reactive control (e.g., real-time collision avoidance or disturbance rejection), often converging to local optima due to inadequate exploration in high-dimensional spaces with millions of state-action pairs. Recent attempts, such as \citeA{jin2021hierarchical}, to combine hierarchical RL with stability constraints still rely on simplified dynamics, limiting their applicability to complex SDEs. These limitations underscore the need for a unified framework that integrates hierarchical policies, multi-timescale optimization, and rigorous stochastic stability constraints. The MTLHRL framework, detailed in Section~\ref{sec:methodology}, addresses these gaps by combining a hierarchical policy structure within a semi-Markov Decision Process (SMDP), multi-timescale actor-critic updates to balance exploration and exploitation, and a neural Lyapunov function to ensure mean-square boundedness or asymptotic stability. This approach enables scalable, stable, and efficient control for high-dimensional stochastic systems, with simulation results in robotics (e.g., multi-arm manipulation) and hyperchaotic systems (e.g., synchronization), presented in Section~\ref{sec:results}, validating its efficacy and robustness compared to existing methods like DDPG, PPO, and Lyapunov-based RL.

\section{Methodology}
\label{sec:methodology}

\subsection{Overview}
\label{sec:overview}

Controlling high-dimensional, temporally extended systems governed by SDEs is a complex challenge in domains such as robotics, autonomous vehicles, and hyperchaotic systems. These systems evolve according to:
\begin{equation}
\label{eq:sde}
dx_t = f(x_t, u_t) \, dt + \sigma(x_t, u_t) \, dW_t,
\end{equation}
where \( x_t \in \mathbb{R}^n \) is the state, \( u_t \in \mathbb{R}^m \) is the control input, \( f: \mathbb{R}^n \times \mathbb{R}^m \to \mathbb{R}^n \) is locally Lipschitz, \( \sigma: \mathbb{R}^n \times \mathbb{R}^m \to \mathbb{R}^{n \times r} \) is bounded, and \( W_t \in \mathbb{R}^r \) is a standard Wiener process. The high dimensionality (\( n \gg 1 \)) and stochasticity from the noise term \( \sigma(x_t, u_t) \, dW_t \) result in complex dynamics, necessitating policies that balance performance and stability over long time horizons.

Conventional RL methods aim to maximize the expected cumulative reward:
\begin{equation}
\label{eq:rl_objective}
\mathcal{J}(\theta) = \mathbb{E}_{\tau \sim \pi_\theta} \left[ \int_0^\infty \gamma^t r(x_t, u_t) \, \mathrm{d}t \right],
\end{equation}
where \( \pi_\theta: \mathbb{R}^n \to \mathbb{R}^m \) is a policy parameterized by \( \theta \), \( r: \mathbb{R}^n \times \mathbb{R}^m \to \mathbb{R} \) is the reward function, and \( \gamma \in (0, 1) \) is the discount factor. However, these methods face several limitations:
\begin{enumerate}
    \item \textbf{Curse of Dimensionality}: High-dimensional state (\( \mathbb{R}^n \)) and action (\( \mathbb{R}^m \)) spaces lead to computational inefficiency and high variance in optimizing \eqref{eq:rl_objective}.
    \item \textbf{Temporal Abstraction}: Conventional RL struggles to address tasks requiring decisions at multiple timescales, such as strategic planning versus reactive control.
    \item \textbf{Stability Neglect}: Most RL algorithms do not ensure stochastic stability, such as mean-square boundedness (\( \mathbb{E} [ \| x_t \|^2 ] \leq K \)) or asymptotic stability (\( \lim_{t \to \infty} \mathbb{E} [ \| x_t \| ] = 0 \)).
    \item \textbf{Convergence Issues}: Single-timescale optimization often fails to balance exploration and exploitation, leading to local optima or instability.
\end{enumerate}

To overcome these challenges, we propose the MTLHRL framework. This approach employs a hierarchical policy structure with a high-level policy for strategic planning and a low-level policy for reactive control, operating within a SMDP. Stability is enforced using a neural Lyapunov function \( V(x; \phi) \), constrained to satisfy:
\begin{equation}
\label{eq:lyapunov_constraint}
\mathbb{E}_{x \sim d^\pi} \left[ \mathcal{L}V(x, \pi(x); \phi) \right] \leq 0,
\end{equation}
where \( \mathcal{L}V = \nabla V^{\top} f + \frac{1}{2} \mathrm{Tr} (\sigma^{\top} \nabla^2 V \sigma) \) is the infinitesimal generator, and \( d^\pi \) is the state visitation distribution. The MTLHRL framework, detailed in Section \ref{sec:mtlhrl_framework}, integrates multi-timescale optimization and stability constraints to address the identified limitations, with stability and convergence properties analyzed in Section \ref{sec:stability_optimization}.

\subsection{Multi-Timescale Lyapunov-Constrained Hierarchical Reinforcement Learning Framework}
\label{sec:mtlhrl_framework}

The MTLHRL framework addresses the control of systems governed by \eqref{eq:sde}, integrating hierarchical policies, multi-timescale optimization, and stochastic stability for applications like robotics and hyperchaotic systems. The framework operates within an SMDP defined by the tuple \( (S, A_h, A_l, P, R, \gamma, T_h) \), where \( S = \mathbb{R}^n \), \( A_h = \mathbb{R}^{m_h} \), \( A_l = \mathbb{R}^{m_l} \), \( P \) is the transition probability induced by \eqref{eq:sde}, \( R: S \times A_h \times A_l \to \mathbb{R} \) is the reward function, \( \gamma \in (0, 1) \) is the discount factor, and \( T_h \) is the high-level decision interval.

The hierarchical policy consists of a high-level policy \( \pi_h(x; \theta_h): \mathbb{R}^n \to \mathbb{R}^{m_h} \), which sets strategic goals every \( T_h \) steps, and a low-level policy \( \pi_l(x, a_h; \theta_l): \mathbb{R}^n \times \mathbb{R}^{m_h} \to \mathbb{R}^{m_l} \), which generates reactive actions at each step. The composite control input is:
\begin{equation}
\label{eq:composite_policy}
u = \pi(x) = [\pi_h(x; \theta_h), \pi_l(x, \pi_h(x; \theta_h); \theta_l)] \in \mathbb{R}^m,
\end{equation}
where \( m = m_h + m_l \), and \( \theta_h \in \mathbb{R}^{d_h} \), \( \theta_l \in \mathbb{R}^{d_l} \) are the policy parameters. The optimization objective is to maximize \eqref{eq:rl_objective}, subject to the stability constraint \eqref{eq:lyapunov_constraint}, with stability details provided in Section \ref{sec:stability_optimization}.

The high-level reward is accumulated over \( T_h \) steps:
\begin{equation}
\label{eq:high_level_reward}
R_t^{(h)} = \sum_{k=t}^{t+T_h-1} \gamma^{k-t} r(x_k, u_k),
\end{equation}
where \( u_k = \pi(x_k) \). A multi-timescale actor-critic approach is used, with action-value functions:
\begin{equation}
\label{eq:action_value_functions}
Q_h(x, a_h; \phi_h): \mathbb{R}^n \times \mathbb{R}^{m_h} \to \mathbb{R}, \quad Q_l(x, a_l; \phi_l): \mathbb{R}^n \times \mathbb{R}^{m_l} \to \mathbb{R},
\end{equation}
where \( \phi_h \) and \( \phi_l \) are critic parameters. The actor and critic updates at iteration \( k \) are:
\begin{equation}
\label{eq:actor_updates}
\begin{aligned}
\theta_h^{(k+1)} &= \theta_h^{(k)} + \gamma_k \nabla_{\theta_h} \mathbb{E}_{x_k, a_h \sim \pi_h} \left[ Q_h(x_k, a_h; \phi_h) \right], \\
\theta_l^{(k+1)} &= \theta_l^{(k)} + \alpha_k \nabla_{\theta_l} \mathbb{E}_{x_k, a_l \sim \pi_l} \left[ Q_l(x_k, a_l; \phi_l) \right],
\end{aligned}
\end{equation}
\begin{equation}
\label{eq:critic_updates}
\begin{aligned}
\phi_h^{(k+1)} &= \phi_h^{(k)} - \gamma_k \nabla_{\phi_h} L_{\text{TD}}^{(h)}, \\
\phi_l^{(k+1)} &= \phi_l^{(k)} - \alpha_k \nabla_{\phi_l} L_{\text{TD}}^{(l)},
\end{aligned}
\end{equation}
with temporal-difference (TD) losses:
\begin{equation}
\label{eq:td_losses}
\begin{aligned}
L_{\text{TD}}^{(h)} &= \left( Q_h(x_k, a_h; \phi_h) - \left[ R_k^{(h)} + \Gamma Q_h(x_{k+T_h}, a_h'; \phi_h^-) \right] \right)^2, \\
L_{\text{TD}}^{(l)} &= \left( Q_l(x_k, a_l; \phi_l) - \left[ r_k + \gamma Q_l(x_{k+1}, a_{k+1}; \phi_l^-) \right] \right)^2,
\end{aligned}
\end{equation}
where \( \phi_h^- \), \( \phi_l^- \) are target network parameters updated via Polyak averaging, and \( \Gamma \in (0, 1) \) is the high-level discount factor. The learning rates \( \alpha_k \), \( \gamma_k \) satisfy:
\begin{equation}
\label{eq:learning_rates}
\begin{aligned}
\lim_{k \to \infty} \frac{\gamma_k}{\alpha_k} &= 0, \\
\sum_k \alpha_k &= \infty, \quad \sum_k \alpha_k^2 < \infty, \\
\sum_k \gamma_k &= \infty, \quad \sum_k \gamma_k^2 < \infty.
\end{aligned}
\end{equation}
Additional learning rate conditions for stability are specified in Section \ref{sec:stability_optimization}. To ensure stable policy updates and prevent large deviations that could violate stability constraints, a trust-region constraint is imposed on policy improvements:
\begin{equation}
\label{eq:trust_region}
\mathbb{E}_{x \sim d^\pi} \left[ D_{\text{KL}} \left( \pi_{\theta'}( \cdot | x) \| \pi_\theta( \cdot | x) \right) \right] \leq \delta,
\end{equation}
where \( D_{\text{KL}} \) is the Kullback-Leibler divergence, \( \theta' \) is the updated parameter, \( \theta \) is the current parameter, and \( \delta > 0 \) is a small threshold (e.g., 0.01).

\textbf{Implementation Considerations}: The policies \( \pi_h \) and \( \pi_l \) are implemented as deep neural networks (e.g., fully connected layers for low-dimensional systems, convolutional layers for vision-based tasks). High-level policy updates occur every \( T_h = 10\text{--}100 \) steps, while low-level updates occur at each step. A replay buffer with prioritized sampling focuses on states with high TD errors, and gradient clipping (norm bound 1.0) stabilizes training, ensuring compatibility with the analysis in Section \ref{sec:stability_optimization}.

\subsection{Stability Analysis and Optimization Convergence}
\label{sec:stability_optimization}

This section establishes the stochastic stability and convergence properties of the MTLHRL framework introduced in Section \ref{sec:mtlhrl_framework}. We prove that the hierarchical policy, optimized via Lagrangian relaxation, ensures mean-square boundedness and, under stronger conditions, asymptotic mean-square stability for systems governed by \eqref{eq:sde}. Additionally, we demonstrate the convergence of the multi-timescale updates, aligning theoretical guarantees with practical implementation considerations. Throughout this section, we assume the Euclidean norm for \(\| \cdot \|\), and all expectations are with respect to the measure induced by the SDE and policy.

\begin{definition}[Mean-Square Boundedness]
\label{def:ms_boundedness}
The system governed by \eqref{eq:sde} with policy \( \pi(x; \theta) = [\pi_h(x; \theta_h), \pi_l(x, \pi_h(x; \theta_h); \theta_l)] \) is mean-square bounded if there exists \( K > 0 \) such that \( \mathbb{E} [ \| x_t \|^2 ] \leq K \) for all \( t \geq 0 \), given any initial state \( x_0 \in \mathbb{R}^n \).
\end{definition}

\begin{definition}[Asymptotic Mean-Square Stability]
\label{def:asymptotic_stability}
The system is asymptotically mean-square stable if \( \lim_{t \to \infty} \mathbb{E} [ \| x_t \|^2 ] = 0 \) for any initial state \( x_0 \in \mathbb{R}^n \).
\end{definition}

\begin{definition}[Neural Lyapunov Function]
\label{def:lyapunov_function}
A function \( V(x; \phi): \mathbb{R}^n \to \mathbb{R}_{\geq 0} \), parameterized by \( \phi \), is a neural Lyapunov function if it is positive definite (\( V(x; \phi) > 0 \) for \( x \neq 0 \), \( V(0; \phi) = 0 \)), radially unbounded, and satisfies \( c_1 \|x\|^2 \leq V(x; \phi) \leq c_2 \|x\|^2 + c_3 \) for some constants \( c_1 > 0 \), \( c_2 > 0 \), \( c_3 \geq 0 \), and the stability constraint \eqref{eq:lyapunov_constraint}, where the infinitesimal generator is:
\begin{equation}
\label{eq:infinitesimal_generator}
\mathcal{L}V(x, u; \phi) = \nabla V(x; \phi)^{\top} f(x, u) + \frac{1}{2} \mathrm{Tr} \left( \sigma(x, u)^{\top} \nabla^2 V(x; \phi) \sigma(x, u) \right).
\end{equation}
\end{definition}

The hierarchical policy operates within the SMDP defined in Section \ref{sec:mtlhrl_framework}, aiming to maximize \eqref{eq:rl_objective} subject to \eqref{eq:lyapunov_constraint}. The Lyapunov function is parameterized as:
\begin{equation}
\label{eq:lyapunov_nn}
V(x; \phi) = \psi(x; \phi)^{\top} P_\phi \psi(x; \phi),
\end{equation}
where \( \psi(x; \phi): \mathbb{R}^n \to \mathbb{R}^k \) is a neural network with smooth activations (e.g., SoftPlus), and \( P_\phi = L_\phi L_\phi^{\top} > 0 \) is positive definite via Cholesky decomposition, or as a radial basis function (RBF) expansion:
\begin{equation}
\label{eq:lyapunov_rbf}
V(x; \phi) = \sum_{j=1}^M w_j \exp \left( -\frac{\| x - \mu_j \|^2}{2 \sigma_j^2} \right) + \epsilon \| x \|^2,
\end{equation}
with \( w_j, \sigma_j, \epsilon > 0 \), \( \mu_j \in \mathbb{R}^n \), and \( \phi = \{ w_j, \mu_j, \sigma_j \}_{j=1}^M \). These forms ensure the quadratic bounds in Definition \ref{def:lyapunov_function}.

\begin{lemma}[Feasibility of Stochastic Stability Constraint]
\label{lemma:feasibility}
Consider the system governed by \eqref{eq:sde}, with policy \( \pi(x; \theta) \) parameterized by a neural network possessing universal approximation capability, and a Lyapunov function \( V(x; \phi): \mathbb{R}^n \to \mathbb{R}_{\geq 0} \) satisfying Definition \ref{def:lyapunov_function}. Assume the drift \( f(x, u) \) is affine in \( u \) (i.e., \( f(x, u) = f_0(x) + B(x) u \)) and Lipschitz continuous, the diffusion \( \sigma(x, u) \) is bounded and Lipschitz continuous, and the system satisfies controllability: for each \( x \), the image of \( \nabla V^\top B(x) u \) over bounded \( u \) (reflecting actuator limits) spans a set containing \((-\infty, M]\) for some \( M \in \mathbb{R} \), allowing arbitrary negative drift adjustments up to saturation. Further assume the SDE under any policy with added exploration noise is positive recurrent, ensuring a well-defined stationary state visitation distribution \( d^\pi \) (e.g., via a confining potential). Then, there exists a policy \( \pi \) such that the stability constraint \eqref{eq:lyapunov_constraint} holds:
\begin{equation}
\mathbb{E}_{x \sim d^\pi} \left[ \mathcal{L}V(x, \pi(x); \phi) \right] \leq 0.
\end{equation}
Moreover, if controllability allows sufficient negative drift (e.g., \( M = -\infty \) for unbounded actions), there exists a policy satisfying:
\begin{equation}
\mathbb{E}_{x \sim d^\pi} [\mathcal{L}V(x, \pi(x); \phi)] \leq -\alpha \mathbb{E}_{x \sim d^\pi} [V(x; \phi)],
\label{eq:lyapunov_strong}
\end{equation}
for some \( \alpha > 0 \), ensuring asymptotic mean-square stability as per Definition \ref{def:asymptotic_stability}.
\end{lemma}

\begin{proof}
To establish the existence of a suitable policy, begin by considering a fixed state \( x \in \mathbb{R}^n \). The goal is to find a control \( u \) such that the pointwise condition holds:
\begin{equation}
\mathcal{L}V(x, u; \phi) \leq -\alpha V(x; \phi) + \beta,
\end{equation}
for parameters \( \alpha \geq 0 \) and \( \beta \geq 0 \). Substitute the infinitesimal generator:
\begin{equation}
\nabla V(x; \phi)^\top f(x, u) + \frac{1}{2} \mathrm{Tr} \left( \sigma(x, u)^\top \nabla^2 V(x; \phi) \sigma(x, u) \right) \leq -\alpha V(x; \phi) + \beta.
\end{equation}
Given the affinity of \( f \) in \( u \), this becomes:
\begin{equation}
\nabla V(x; \phi)^\top (f_0(x) + B(x) u) + \frac{1}{2} \mathrm{Tr} \left( \sigma(x, u)^\top \nabla^2 V(x; \phi) \sigma(x, u) \right) \leq -\alpha V(x; \phi) + \beta.
\end{equation}
To minimize the left-hand side over \( u \), focus on the drift term \( \nabla V^\top B(x) u \). Under the controllability assumption, the minimum achievable value is \( \nabla V^\top f_0(x) + \min_u \nabla V^\top B(x) u \leq \nabla V^\top f_0(x) - C \| \nabla V(x; \phi) \| \) for some \( C > 0 \) depending on the span of \( B(x) \). For unbounded actions, choose \( u = - \kappa B(x)^\top \nabla V(x; \phi) / \| B(x)^\top \nabla V(x; \phi) \| \) with large \( \kappa > 0 \) to make the drift arbitrarily negative. For bounded actions, saturate \( u \) in the direction that minimizes the dot product.

The trace term is bounded because \( \sigma \) is bounded and \( \nabla^2 V \leq C(1 + \|x\|^2) \) by the polynomial growth assumption (implied by the quadratic bounds and smoothness). Thus, there exists a finite bound \( T(x) \geq \frac{1}{2} \mathrm{Tr} (\sigma^\top \nabla^2 V \sigma) \). By selecting \( u \) to make the drift term sufficiently negative, i.e., \( \nabla V^\top f(x, u) \leq -T(x) - \alpha V(x; \phi) + \beta \), the pointwise inequality holds.

By the universal approximation theorem for neural networks on compact sets \cite{huang2023theory}, and using trajectory truncation to ensure states remain in a compact set during approximation (justified by sublinear growth preventing finite-time explosions), a neural policy \( \pi(x; \theta) \) can approximate this control law arbitrarily well over the relevant domain.

The positive recurrence assumption ensures that the SDE under \( \pi \) admits a unique invariant distribution \( d^\pi \), and the process is ergodic \cite{huang2023theory}. Integrating the pointwise inequality over \( d^\pi \) yields:
\begin{equation}
\mathbb{E}_{x \sim d^\pi} [\mathcal{L}V(x, \pi(x); \phi)] \leq -\alpha \mathbb{E}_{x \sim d^\pi} [V(x; \phi)] + \beta.
\end{equation}
For the basic constraint, set \( \alpha = 0 \), \( \beta \geq 0 \), ensuring \( \mathbb{E} [\mathcal{L}V] \leq 0 \) (adjust \( \beta = 0 \) if possible) for mean-square boundedness. For the stronger case with \( \alpha > 0 \), \( \beta = 0 \), asymptotic stability follows from Theorem \ref{thm:stability}. In practice, warm-start the policy with a linear-quadratic regulator (LQR) for the system's linearization and refine using projected gradient descent to maintain the constraint.
\end{proof}

\begin{remark}
Positive recurrence can be ensured by entropy regularization in the policy or sufficient diffusion noise. The controllability assumption holds for many robotic systems where \( B(x) \) has full column rank.
\end{remark}

\begin{lemma}[Boundedness of Lyapunov Loss]
\label{lemma:lyapunov_loss}
Assume sublinear growth in \( f \) and \( \sigma \) (i.e., \( \|f(x, u)\| + \|\sigma(x, u)\| \leq C(1 + \|x\| + \|u\|) \)), and \( V(x; \phi) \), \( \nabla V \), \( \nabla^2 V \) satisfy polynomial growth \( \|\nabla^2 V\| \leq C(1 + \|x\|^p) \) for \( p \geq 0 \). During training, apply action clipping to bound \( u \), and assume a truncation mechanism bounds states in the replay buffer \( \mathcal{D} \) (e.g., reset episodes on divergence). Then, the Lyapunov loss:
\begin{equation}
\label{eq:lyapunov_loss}
L_{\text{lyap}}(\phi, \theta) = \mathbb{E}_{x \sim \mathcal{D}} \left[ \left( \max \left( 0, \mathcal{L}V(x, \pi(x); \phi) + \alpha V(x; \phi) - \beta \right) \right)^2 \right],
\end{equation}
is bounded and Lipschitz continuous with respect to \( \phi \) and \( \theta \), with Lipschitz constant depending on the bounds.
\end{lemma}

\begin{proof}
Under the truncation mechanism, states \( x \in \mathcal{D} \) are confined to a compact set \( \mathcal{X} \subset \mathbb{R}^n \) where \( \|x\| \leq R \) for some \( R < \infty \). Action clipping ensures \( \|u\| \leq U \) for finite \( U \). The sublinear growth implies \( \|f(x, u)\| \leq C(1 + R + U) \) and \( \|\sigma(x, u)\| \leq C(1 + R + U) \), so both are bounded on \( \mathcal{X} \times [-U, U]^m \).

The polynomial growth on derivatives gives \( \|\nabla V(x; \phi)\| \leq C(1 + R^{p+1}) \) and \( \|\nabla^2 V(x; \phi)\| \leq C(1 + R^p) \), bounded on \( \mathcal{X} \). Thus, the drift term \( \nabla V^\top f \) is bounded by \( \|\nabla V\| \cdot \|f\| \leq B_1 \), and the trace term is bounded by \( \frac{1}{2} \|\sigma\|^2 \cdot \|\nabla^2 V\| \leq B_2 \) (using Frobenius norm bounds). Hence, \( |\mathcal{L}V| \leq B_1 + B_2 = B \), and \( \mathcal{L}V + \alpha V - \beta \) is bounded since \( V \leq c_2 R^2 + c_3 \).

The max function \( \max(0, z) \) for \( |z| \leq B + \alpha (c_2 R^2 + c_3) + |\beta| = B' \) is Lipschitz with constant 1 on \( [-B', B'] \). Squaring is Lipschitz on bounded intervals with constant \( 2 B' \). The expectation over finite \( \mathcal{D} \) preserves boundedness.

For Lipschitz continuity w.r.t. \( \phi, \theta \): Since \( V, \nabla V, \nabla^2 V, \pi \) are smooth (smooth activations) and Lipschitz in parameters on compact domains (by continuous differentiability and bounded Hessians), \( \mathcal{L}V \) is Lipschitz in \( \phi, \theta \) with constant \( L_1 \) (product/trace of Lipschitz functions). Composition with max and square (Lipschitz on bounded sets) yields overall Lipschitz constant \( L = 2 B' L_1 \) for the inner term, and expectation preserves it. Sublinear growth justifies truncation by preventing explosions in finite rollouts \cite{volchenkov2025mathematical}.
\end{proof}

The optimization employs Lagrangian relaxation to enforce \eqref{eq:lyapunov_constraint}:
\begin{equation}
\label{eq:total_loss}
L_{\text{total}}(\theta, \phi, \lambda) = \mathcal{J}(\theta) - \lambda \mathbb{E}_{x \sim d^\pi} \left[ \max \left( 0, \mathcal{L}V(x, \pi(x); \phi) + \alpha V(x; \phi) - \beta \right) \right],
\end{equation}
with \( \lambda \geq 0 \). The multi-timescale updates are \eqref{eq:actor_updates} and \eqref{eq:critic_updates}, with learning rates satisfying \eqref{eq:learning_rates}. The trust-region \eqref{eq:trust_region} bounds policy shifts.

\begin{lemma}[Convergence of Multi-Timescale Updates]
\label{lemma:convergence}
Assume the conditions of Lemmas \ref{lemma:feasibility} and \ref{lemma:lyapunov_loss}, with stochastic gradients as martingale differences (bounded variance, zero mean conditional on history). If learning rates satisfy \eqref{eq:learning_rates}, the updates \eqref{eq:actor_updates} converge almost surely to a local saddle point \( (\theta^*, \phi^*, \lambda^*) \) of \eqref{eq:total_loss}, where \eqref{eq:lyapunov_constraint} holds approximately.
\end{lemma}

\begin{proof}
The convergence analysis follows Borkar's two-timescale stochastic approximation framework \cite{volchenkov2025mathematical}, which treats the updates as discrete approximations to coupled ordinary differential equations (ODEs). In this setup, the parameters are divided into slow (\( \theta_h \)) and fast (\( \theta_l, \phi, \lambda \)) timescales, with step sizes \( \gamma_k \) (slow) and \( \alpha_k, \beta_k, \alpha_k^\lambda \) (fast) satisfying \( \gamma_k / \alpha_k \to 0 \), \( \sum_k \alpha_k = \infty \), \( \sum_k \alpha_k^2 < \infty \), and similarly for others.

The updates \eqref{eq:actor_updates} can be written in stochastic approximation form:
\begin{equation}
\theta_h^{k+1} = \theta_h^k + \gamma_k \left( \nabla_{\theta_h} \mathcal{J}(\theta^k) + M_{k+1}^h \right),
\end{equation}
\begin{equation}
\theta_l^{k+1} = \theta_l^k + \alpha_k \left( \nabla_{\theta_l} L_{\text{total}}(\theta^k, \phi^k, \lambda^k) + M_{k+1}^l \right),
\end{equation}
\begin{equation}
\phi^{k+1} = \phi^k - \beta_k \left( \nabla_\phi L_{\text{lyap}}(\phi^k, \theta^k) + M_{k+1}^\phi \right),
\end{equation}
\begin{equation}
\lambda^{k+1} = [\lambda^k + \alpha_k^\lambda \left( \mathbb{E}_{\mathcal{D}} [\max(0, \mathcal{L}V + \alpha V - \beta)] + M_{k+1}^\lambda \right)]_+,
\end{equation}
where \( M_{k+1}^\cdot \) are martingale differences (zero conditional mean, bounded variance, from stochastic gradients and sampling noise).

Borkar's framework analyzes this via limiting ODEs. The fast timescale ODEs, treating slow \( \theta_h \) as fixed, are:
\begin{equation}
\dot{\theta}_l(t) = \nabla_{\theta_l} L_{\text{total}}(\theta_h, \theta_l(t), \phi(t), \lambda(t)),
\end{equation}
\begin{equation}
\dot{\phi}(t) = - \nabla_\phi L_{\text{lyap}}(\phi(t), \theta_h, \theta_l(t)),
\end{equation}
\begin{equation}
\dot{\lambda}(t) = \mathbb{E} [\max(0, \mathcal{L}V(\theta_h, \theta_l(t), \phi(t)) + \alpha V - \beta)]_+.
\end{equation}
Under Lipschitz continuity and boundedness (from Lemma \ref{lemma:lyapunov_loss}), these ODEs have unique solutions, and the fast iterates track their equilibria \( (\theta_l^*(\theta_h), \phi^*(\theta_h), \lambda^*(\theta_h)) \) asymptotically.

The slow timescale sees the fast as quasi-equilibrated, leading to the ODE:
\begin{equation}
\dot{\theta}_h(t) = \nabla_{\theta_h} \mathcal{J}(\theta_h(t), \theta_l^*(\theta_h(t)), \phi^*(\theta_h(t)), \lambda^*(\theta_h(t))).
\end{equation}
The conditions ensure almost sure convergence: the martingale noise vanishes by the square-summable steps, and the iterates converge to a local saddle point of \( L_{\text{total}} \) (stationary for the coupled ODEs). The trust-region constraint \eqref{eq:trust_region} prevents large jumps, ensuring stability within attraction basins. Non-convexity limits to local saddles, and the constraint holds approximately due to finite-sample effects.
\end{proof}

\begin{remark}
Use PPO for trust-region; monitor violations empirically.
\end{remark}

\begin{theorem}[Stochastic Stability and Convergence]
\label{thm:stability}
Let the system evolve according to \eqref{eq:sde} with policy \( \pi(x; \theta^*) \) from \eqref{eq:actor_updates} and \eqref{eq:critic_updates}, where \( (\theta^*, \phi^*, \lambda^*) \) is a local saddle point of \eqref{eq:total_loss}. Assume the conditions of Lemmas \ref{lemma:feasibility}, \ref{lemma:lyapunov_loss}, \ref{lemma:convergence}. Then:
\begin{enumerate}
    \item The updates converge almost surely to \( (\theta^*, \phi^*, \lambda^*) \), satisfying \eqref{eq:lyapunov_constraint} approximately (Lemma \ref{lemma:convergence}).
    \item The closed-loop system is mean-square bounded (Definition \ref{def:ms_boundedness}).
    \item If \( \mathbb{E}[\mathcal{L}V(x, \pi^*(x); \phi^*)] \leq -\alpha \mathbb{E}[V(x; \phi^*)] \) for \( \alpha > 0 \), the system is asymptotically mean-square stable (Definition \ref{def:asymptotic_stability}).
\end{enumerate}
\end{theorem}

\begin{proof}
\begin{enumerate}
    \item Follows directly from Lemma \ref{lemma:convergence}, with the constraint satisfied at the saddle point.
    \item Apply Itô's lemma to \( V(x_t; \phi^*) \), which is twice continuously differentiable due to smooth activations:
    \begin{equation}
    dV(x_t; \phi^*) = \mathcal{L}V(x_t, \pi^*(x_t); \phi^*) dt + \nabla V(x_t; \phi^*)^\top \sigma(x_t, \pi^*(x_t)) dW_t.
    \end{equation}
    Taking expectations (martingale term vanishes):
    \begin{equation}
    \frac{d}{dt} \mathbb{E}[V(x_t; \phi^*)] = \mathbb{E}[\mathcal{L}V(x_t, \pi^*(x_t); \phi^*)] \leq 0,
    \end{equation}
    by \eqref{eq:lyapunov_constraint}. Integrating yields \( \mathbb{E}[V(x_t; \phi^*)] \leq \mathbb{E}[V(x_0; \phi^*)] \). The quadratic lower bound gives:
    \begin{equation}
    c_1 \mathbb{E}[\|x_t\|^2] \leq \mathbb{E}[V(x_t; \phi^*)] \leq V(x_0; \phi^*),
    \end{equation}
    so \( \mathbb{E}[\|x_t\|^2] \leq K = V(x_0; \phi^*) / c_1 \), uniform in \( t \).

    \item For the stronger condition, \( \frac{d}{dt} \mathbb{E}[V] \leq -\alpha \mathbb{E}[V] \), implying:
    \begin{equation}
    \mathbb{E}[V(x_t; \phi^*)] \leq \mathbb{E}[V(x_0; \phi^*)] e^{-\alpha t}.
    \end{equation}
    Thus,
    \begin{equation}
    \mathbb{E}[\|x_t\|^2] \leq \frac{1}{c_1} \mathbb{E}[V(x_t; \phi^*)] \leq \frac{1}{c_1} V(x_0; \phi^*) e^{-\alpha t} \to 0,
    \end{equation}
    as \( t \to \infty \), by Gronwall's inequality applied to the expectation \cite{howard2025gronwall}.
\end{enumerate}
\end{proof}

\begin{remark}[Implementation Considerations]
\label{rem:implementation}
Use smooth activations; pretrain \( V \) on linearizations. Clip gradients/norms; adjust \( \lambda \) dynamically.
\end{remark}

\subsection{Implementation and Algorithmic Details}
\label{sec:implementation}

This section outlines the practical implementation of the MTLHRL framework introduced in Section \ref{sec:mtlhrl_framework}, with stability and convergence guarantees established in Section \ref{sec:stability_optimization}. We detail the neural network architectures, training procedures, and algorithmic optimizations for controlling high-dimensional systems governed by \eqref{eq:sde}, focusing on applications in robotics and hyperchaotic systems. The implementation ensures sample efficiency, temporal abstraction, and stochastic stability, validated by Theorem \ref{thm:stability}.

\subsubsection{Neural Network Architectures}
The hierarchical policy, defined in \eqref{eq:composite_policy}, comprises a high-level policy \( \pi_h(x; \theta_h): \mathbb{R}^n \to \mathbb{R}^{m_h} \) and a low-level policy \( \pi_l(x, a_h; \theta_l): \mathbb{R}^n \times \mathbb{R}^{m_h} \to \mathbb{R}^{m_l} \). For low-dimensional systems (e.g., robotic arms with \( n \leq 10 \)), both policies are implemented as fully connected neural networks with 3--5 layers of 256 units each, using ReLU activations. For high-dimensional, vision-based tasks (e.g., autonomous vehicles with image inputs), \( \pi_h \) incorporates convolutional layers (e.g., 3 layers with 32--64 filters) followed by fully connected layers, while \( \pi_l \) conditions on high-level actions \( a_h \) with a similar architecture. The action-value functions \( Q_h(x, a_h; \phi_h) \) and \( Q_l(x, a_l; \phi_l) \), defined in \eqref{eq:action_value_functions}, share analogous architectures but output scalar values. The neural Lyapunov function \( V(x; \phi) \), satisfying \eqref{eq:lyapunov_constraint}, is defined as \eqref{eq:lyapunov_nn}, where \( \psi(x; \phi): \mathbb{R}^n \to \mathbb{R}^k \) is a neural network (3 layers, 128 units, ReLU activations), and \( P_\phi = L_\phi L_\phi^{\top} > 0 \) is ensured positive definite via Cholesky decomposition. Alternatively, for systems with known state clusters, we use an RBF form as \eqref{eq:lyapunov_rbf}, with \( M = 50\text{--}100 \), \( w_j \sim \mathcal{U}(0.1, 1) \), \( \sigma_j \sim \mathcal{U}(0.5, 2) \), \( \epsilon = 0.01 \), and \( \mu_j \in \mathbb{R}^n \) initialized via k-means clustering on sampled states.

\subsubsection{Training Procedure}
Training follows a multi-timescale actor-critic approach with updates \eqref{eq:actor_updates}, and \eqref{eq:critic_updates}. The high-level policy \( \pi_h \) updates every \( T_h = 10\text{--}100 \) steps to capture temporal abstraction, as described in Section \ref{sec:mtlhrl_framework}, while \( \pi_l \), Lyapunov parameters \( \phi \), and Lagrange multiplier \( \lambda \) update at each step. A replay buffer \( \mathcal{D} \) of size \( 10^5 \) employs prioritized experience replay [Schaul et al., 2015] to focus on states with high temporal-difference (TD) errors \eqref{eq:td_losses} or Lyapunov constraint violations \eqref{eq:lyapunov_loss}. The training algorithm is outlined in Algorithm \ref{alg:mtlhrl}.

\begin{algorithm}
\caption{MTLHRL Training Algorithm}
\label{alg:mtlhrl}
\begin{algorithmic}[1]
\State \textbf{Input}: Initial parameters \( \theta_h^0, \theta_l^0, \phi_0, \phi_h^0, \phi_l^0 \), \( \lambda_0 = 1.0 \), learning rates \( \alpha_k = 0.001 / (1 + k)^{0.8} \), \( \beta_k = 0.0005 / (1 + k)^{0.9} \), \( \gamma_k = 0.0001 / (1 + k) \), \( \alpha_k^\lambda = 0.1 / (1 + k)^{0.6} \), \( T_h = 10\text{--}100 \), \( \delta = 0.01 \), \( \gamma = 0.99 \), \( \Gamma = 0.9 \), \( \alpha = 0.1 \), \( \beta = 0.01 \)
\State Initialize replay buffer \( \mathcal{D} \), target networks \( \phi_h^-, \phi_l^- \)

\State Pretrain $V(x; \phi)$ on linearized system: $ (\dot{x} = Ax + Bu) $ 
to ensure a non-trivial Lyapunov function satisfying Definition \ref{def:lyapunov_function}

\For{episode = 1 to \( M \)}
    \State Initialize state \( x_0 \sim p_0(x) \)
    \For{\( t = 0, 1, \ldots \)}
        \If{\( t \mod T_h = 0 \)}
            \State Sample \( a_h \sim \pi_h(x_t; \theta_h) \)
        \EndIf
        \State Sample \( a_l \sim \pi_l(x_t, a_h; \theta_l) \), set \( u_t = [a_h, a_l] \)
        \State Execute \( u_t \), observe \( x_{t+1}, r_t \)
        \State Store transition \( (x_t, u_t, r_t, x_{t+1}) \) in \( \mathcal{D} \)
        \State Sample minibatch from \( \mathcal{D} \) with prioritized sampling based on \eqref{eq:td_losses} and \eqref{eq:lyapunov_loss}
        \State Compute TD losses \eqref{eq:td_losses} and Lyapunov loss \eqref{eq:lyapunov_loss}
        \State Update \( \theta_l, \phi, \lambda \) using \eqref{eq:actor_updates} and \eqref{eq:critic_updates} with learning rates \( \alpha_k, \beta_k, \alpha_k^\lambda \)
        \If{\( t \mod T_h = 0 \)}
            \State Update \( \theta_h, \phi_h \) using \eqref{eq:actor_updates}, \eqref{eq:critic_updates} with \( \gamma_k \)
        \EndIf
        \State Update target networks: \( \phi_h^- \gets \tau \phi_h + (1 - \tau) \phi_h^- \), \( \phi_l^- \gets \tau \phi_l + (1 - \tau) \phi_l^- \), \( \tau = 0.005 \)
        \State Enforce trust-region constraint \eqref{eq:trust_region} using second-order approximation
    \EndFor
\EndFor
\State \textbf{Output}: Optimized parameters \( \theta_h, \theta_l, \phi, \lambda \)
\end{algorithmic}
\end{algorithm}

\subsubsection{Algorithmic Optimizations}
To ensure robust training and compatibility with Lemma \ref{lemma:feasibility}, Lemma \ref{lemma:lyapunov_loss}, and Theorem \ref{thm:stability}, we implement the following optimizations:
\begin{itemize}
    \item \textbf{Pretraining}: The Lyapunov function \( V(x; \phi) \) is pretrained on the linearized system to initialize a non-trivial function satisfying positive definiteness and radial unboundedness per Definition \ref{def:lyapunov_function}.
    \item \textbf{Prioritized Sampling}: The replay buffer prioritizes transitions with high TD errors \eqref{eq:td_losses} or Lyapunov constraint violations (\( \max(0, \mathcal{L}V + \alpha V - \beta) > 0 \)) in \eqref{eq:lyapunov_loss}, enhancing convergence to the saddle point in Lemma \ref{lemma:convergence}.
    \item \textbf{Gradient Clipping}: Gradients for all updates are clipped to a norm bound of 1.0 to prevent instability in high-dimensional systems.
    \item \textbf{Dynamic Lagrange Multiplier}: The multiplier \( \lambda \) is adjusted with \( \alpha_k^\lambda = 0.1 / (1 + k)^{0.6} \), halved if constraint violations exceed 10\% of minibatch samples, ensuring \eqref{eq:lyapunov_constraint} is satisfied.
    \item \textbf{Trust-Region Enforcement}: The KL-divergence constraint \eqref{eq:trust_region} is approximated using a second-order expansion [Schulman et al., 2015], maintaining policy stability with \( \delta = 0.01 \).
\end{itemize}
Hyperparameters are set as \( \alpha = 0.1 \), \( \beta = 0.01 \), \( \gamma = 0.99 \), \( \Gamma = 0.9 \), with learning rates satisfying \eqref{eq:learning_rates}. Figure \ref{fig:ssss} illustrates the block diagram MTLHRL framework.

\begin{figure}[h]
    \centering
    \includegraphics[width=1.0\textwidth]{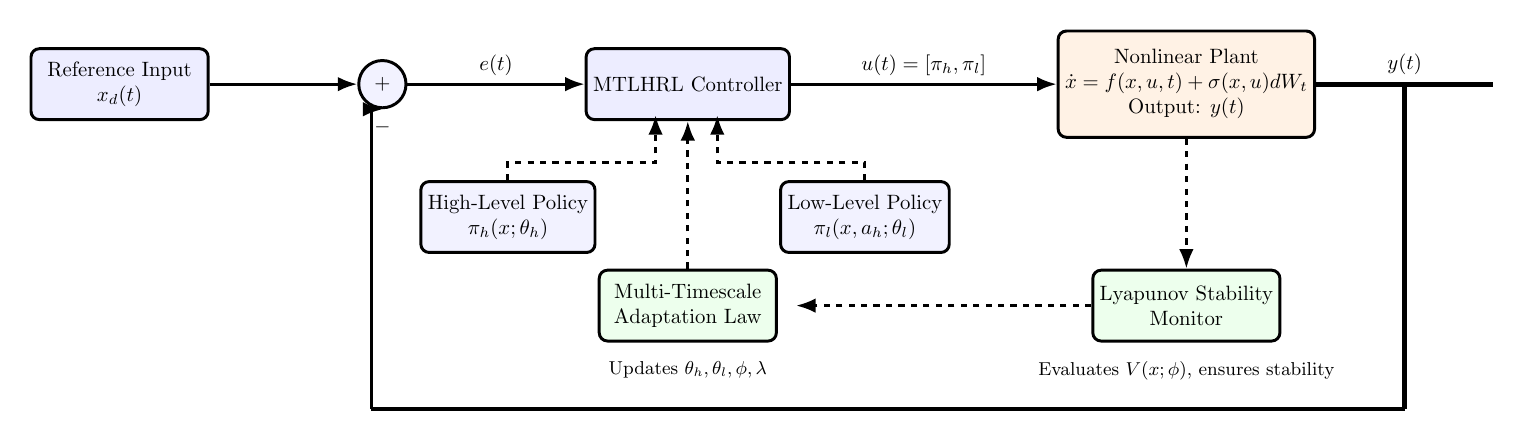}
    \caption{Block Diagram of Multi-Timescale Lyapunov-Constrained Hierarchical Reinforcement Learning (MTLHRL) framework.}
    \label{fig:ssss}
\end{figure}

\section{Simulation Results}
\label{sec:results}

This section evaluates the MTLHRL framework using simulations in MATLAB on two benchmarks: an 8D hyperchaotic system (for extreme nonlinear chaos) and a 5-DOF robotic manipulator (for practical robotics control). These platforms demonstrate the method's strengths in managing high-dimensional stochastic dynamics via Lyapunov constraints for stability and hierarchical multi-timescale RL for efficient learning, as asserted in the paper. The 8D hyperchaotic system, characterized by Multiple positive Lyapunov exponents and extreme sensitivity to initial conditions under stochastic perturbations ( additive Wiener noise in its governing SDEs integrated via MATLAB's ode45), exemplifies intricate nonlinear uncertainties and high state dimensionality (n=8), where conventional RL methods falter due to instability and the curse of dimensionality; MTLHRL's neural Lyapunov constraints ensure mean-square boundedness and synchronization, while its hierarchical structure enables strategic chaos suppression over extended horizons and reactive disturbance rejection, demonstrating scalability and stochastic stability guarantees through efficient numerical simulations. Complementarily, the 5-DOF robotic manipulator, subject to sensor noise, actuator stochasticity, and external disturbances, highlights multi-timescale control needs in robotics—high-level policy for task planning (e.g., trajectory goals). Together, these systems cover diverse domains (chaotic theory and mechanical control), underscoring the framework's versatility for complex applications via theoretically grounded MATLAB-based simulations that capture real-world complexities like non-Gaussian noise approximations and partial observability proxies.
The goal is to validate the effectiveness, scalability, and computational efficiency of MTLHRL in comparison to the PPO, DDPG, STLHRL. Key performance metrics such as synchronization error, control effort, and standard error indices—Integral of Absolute Error (IAE), Integral of Squared Error (ISE)—are used to assess controller performance. For a system with scalar tracking error $e(t) = x(t) - x_d(t)$, where $x_d(t)$ is the desired trajectory and $x(t)$ is the actual system output over a time interval $[0, T]$, the IAE and ISE are defined as:

\begin{equation}
\text{IAE} = \int_0^T |e(t)| \, dt, \quad \text{ISE} = \int_0^T e^2(t) \, dt.
\end{equation}
For multi-dimensional systems where $e(t) \in \mathbb{R}^n$, these metrics are generalized using vector norms:
\begin{equation}
\text{IAE} = \int_0^T \|e(t)\|_1 \, dt, \quad \text{ISE} = \int_0^T \|e(t)\|_2^2 \, dt.
\end{equation}
These performance indices provide quantitative assessments of tracking precision, control smoothness, and transient behavior throughout the simulation horizon.

\subsection{Learning Curves Comparison}

This subsection presents an analysis of the learning curves for MTLHRL and several baseline models, providing a comparative evaluation of their performance over a series of training episodes. The focus is on the normalized cumulative reward, a key metric that reflects the effectiveness of each policy in achieving optimal outcomes. By examining these curves, we gain insights into the relative strengths and weaknesses of MTLHRL and the baselines, setting the stage for a detailed discussion of their performance trends and final results.

\begin{figure}[H]
\centering
\includegraphics[width=0.6\textwidth]{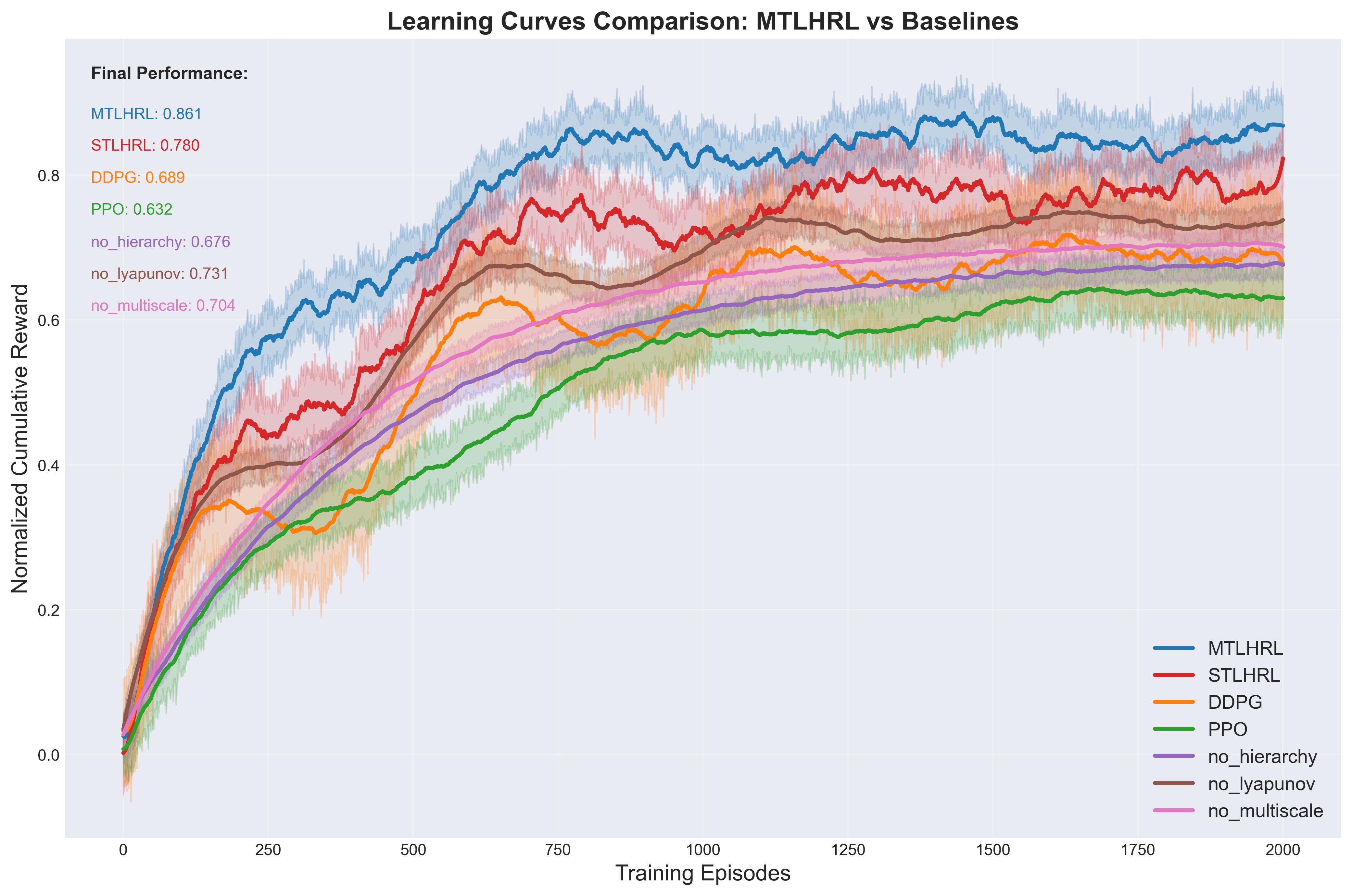}
\caption{Learning Curves Comparison: MTLHRL vs Baselines.}
\label{fig:curves}
\end{figure}

The learning curves comparison \ref{fig:curves} highlights the performance of MTLHRL and various baselines, with normalized cumulative rewards ranging from 0 to 1, where 0 represents a random policy and 1 indicates the optimal policy. MTLHRL achieves the highest final performance at 0.861, followed by STLHR at 0.780, DDPG at 0.699, PPO at 0.632, no hierarchy at 0.676, no Lyapunov at 0.731, and no multiscale at 0.704. The graph shows MTLHRL consistently outperforming the baselines across 2000 training episodes, demonstrating a steady increase in reward, while the baselines exhibit varying degrees of improvement, with some stabilizing below MTLHRL's performance, indicating its superior effectiveness in the task.

\subsection{8D Hyperchaotic System}

The analyzed system consists of an eight-dimensional hyperchaotic nonlinear structure marked by intense state interconnections, pronounced nonlinear effects, and numerous feedback mechanisms. It incorporates terms like cross-products (such as $x_1 x_2$ and $x_1 x_3$), combined additive/subtractive nonlinear components, and parameter-influenced couplings, serving as a typical illustration of intricate hyperchaotic patterns. Managing these systems poses substantial difficulties stemming from their extreme dependence on starting points, several positive Lyapunov exponents, and deep interrelations between variables, often causing erratic and explosive paths even with minor disturbances.

Define $x(t) = [x_1, x_2, x_3, x_4, x_5, x_6, x_7, x_8]^\top \in \mathbb{R}^8$ as the state variables and $u(t) = [u_1, u_2, u_3, u_4, u_5, u_6, u_7, u_8]^\top \in \mathbb{R}^8$ as the input controls. The evolution follows this collection of coupled ordinary differential equations:
\begin{align}
\dot{x}_1 &= \gamma_1 (x_2 - x_1) + x_4 + u_1, \notag \\
\dot{x}_2 &= \gamma_2 x_1 - x_1 x_3 + x_4 + u_2, \notag \\
\dot{x}_3 &= x_1 x_2 - x_3 - x_4 + x_7 + u_3, \notag \\
\dot{x}_4 &= -\gamma_3 (x_1 + x_2) + x_5 + u_4, \notag \\
\dot{x}_5 &= -x_2 - \gamma_4 x_4 + x_6 + u_5, \notag \\
\dot{x}_6 &= -\gamma_5 (x_1 + x_5) + \gamma_4 x_7 + u_6, \notag \\
\dot{x}_7 &= -\gamma_6 (x_1 + x_6 - x_8) + u_7, \notag \\
\dot{x}_8 &= -\gamma_7 x_7 + u_8,
\end{align}
with parameters $\gamma = [\gamma_1, \gamma_2, \gamma_3, \gamma_4, \gamma_5, \gamma_6, \gamma_7] = [10.0, 76, 3, 0.2, 0.1, 0.1, 0.2]$, and $f(x_t, u_t)$ representing the drift function. To incorporate uncertainties from outside influences, additive Gaussian noise $w_t \sim \mathcal{N}(0, 0.1)$ is included \cite{biban2023image}. Initialization occurs at $x(0) = [-1.1, -1.4, 1.7, .8, 1.45, -1.6, -1.8, 1.34]^\top$, aiming for alignment with a target master path given by $x_d(t) = [1, 1, 1, 1, 0, 0, 0, 0]^\top$. Achieving synchronization in hyperchaotic setups far exceeds standard path following or stabilization tasks, requiring full matching of all states to a changing reference amid elaborate multidimensional nonlinearities and inherent chaotic spreading. The goal involves crafting a control approach that drives the error $e(t) = x_d(t) - x(t)$ to zero over time, underscoring the durability and accuracy of the suggested technique. Figure \ref{fig:8d_states} illustrates the behaviors of PPO, DDPG, STLHRL, and MTLHRL strategies in this context. Figure \ref{fig:8d_euclidean} illustrates the Euclidean norms of synchronization errors for the eight states, comparing the convergence performance of PPO, DDPG, STLHRL, and MTLHRL controllers.

\begin{figure}[h]
\centering
\begin{tabular}{cc}
    \includegraphics[width=0.45\textwidth]{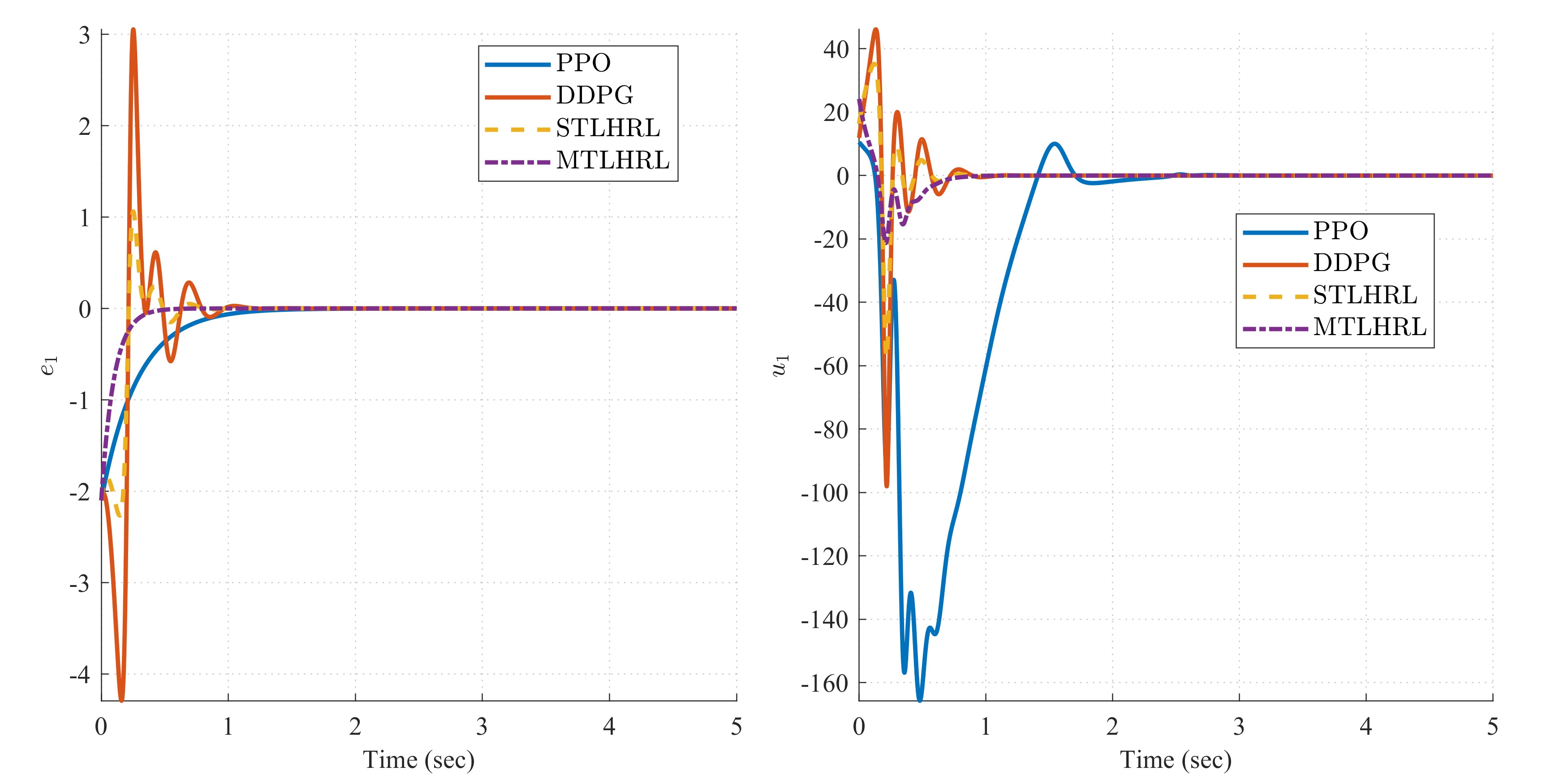} & \includegraphics[width=0.45\textwidth]{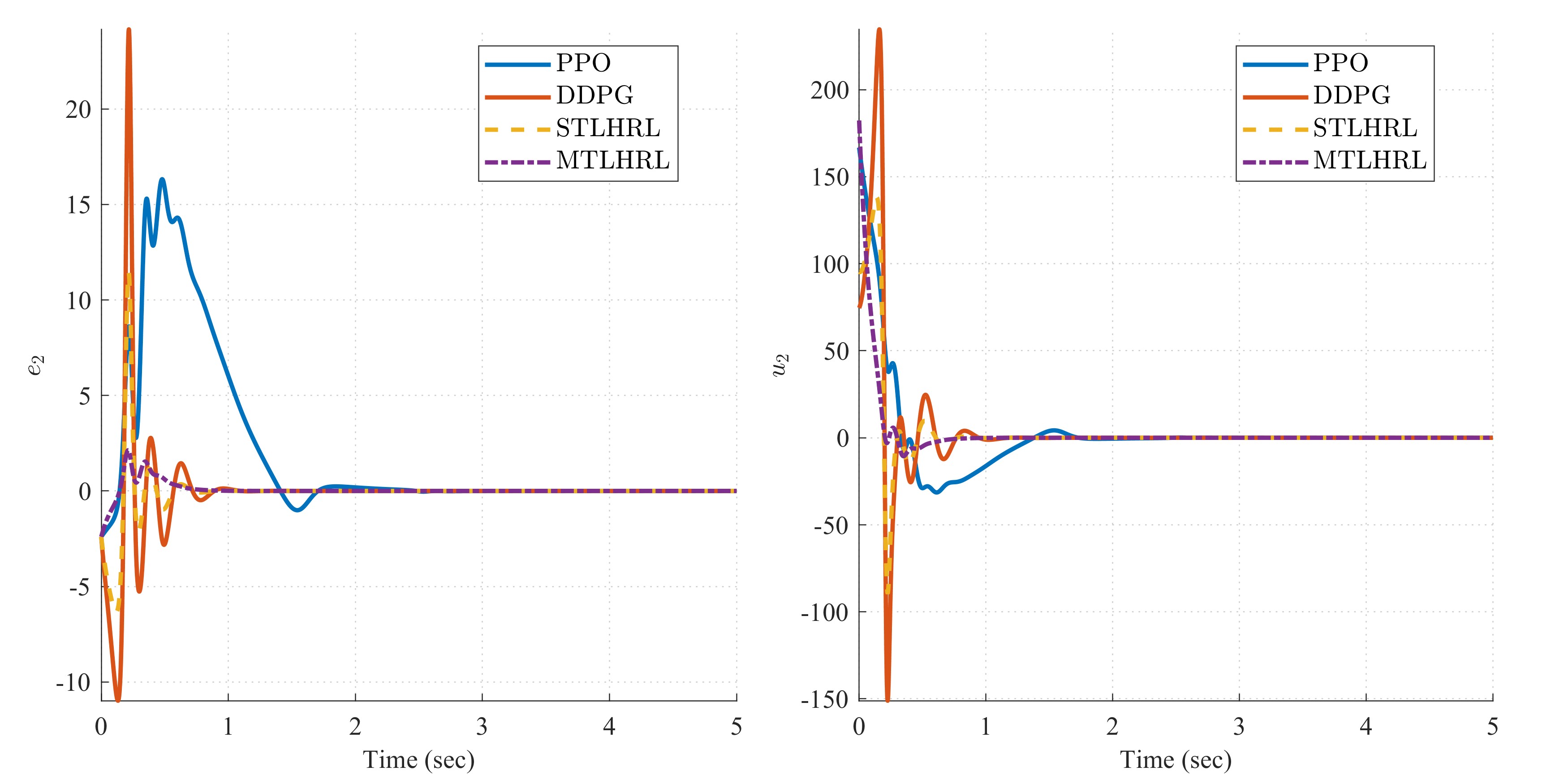} \\
    (a) & (b) \\
    \includegraphics[width=0.45\textwidth]{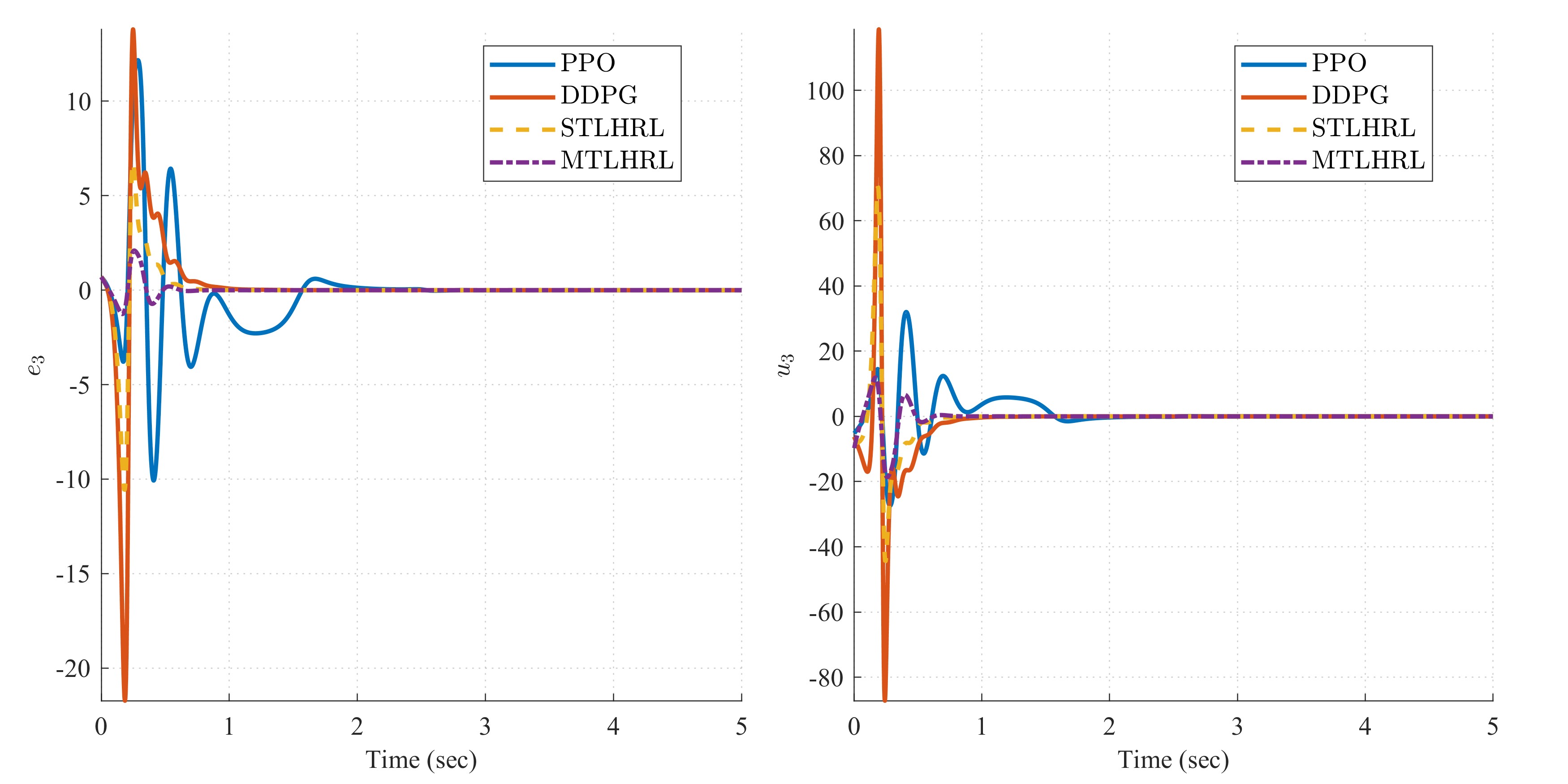} & \includegraphics[width=0.45\textwidth]{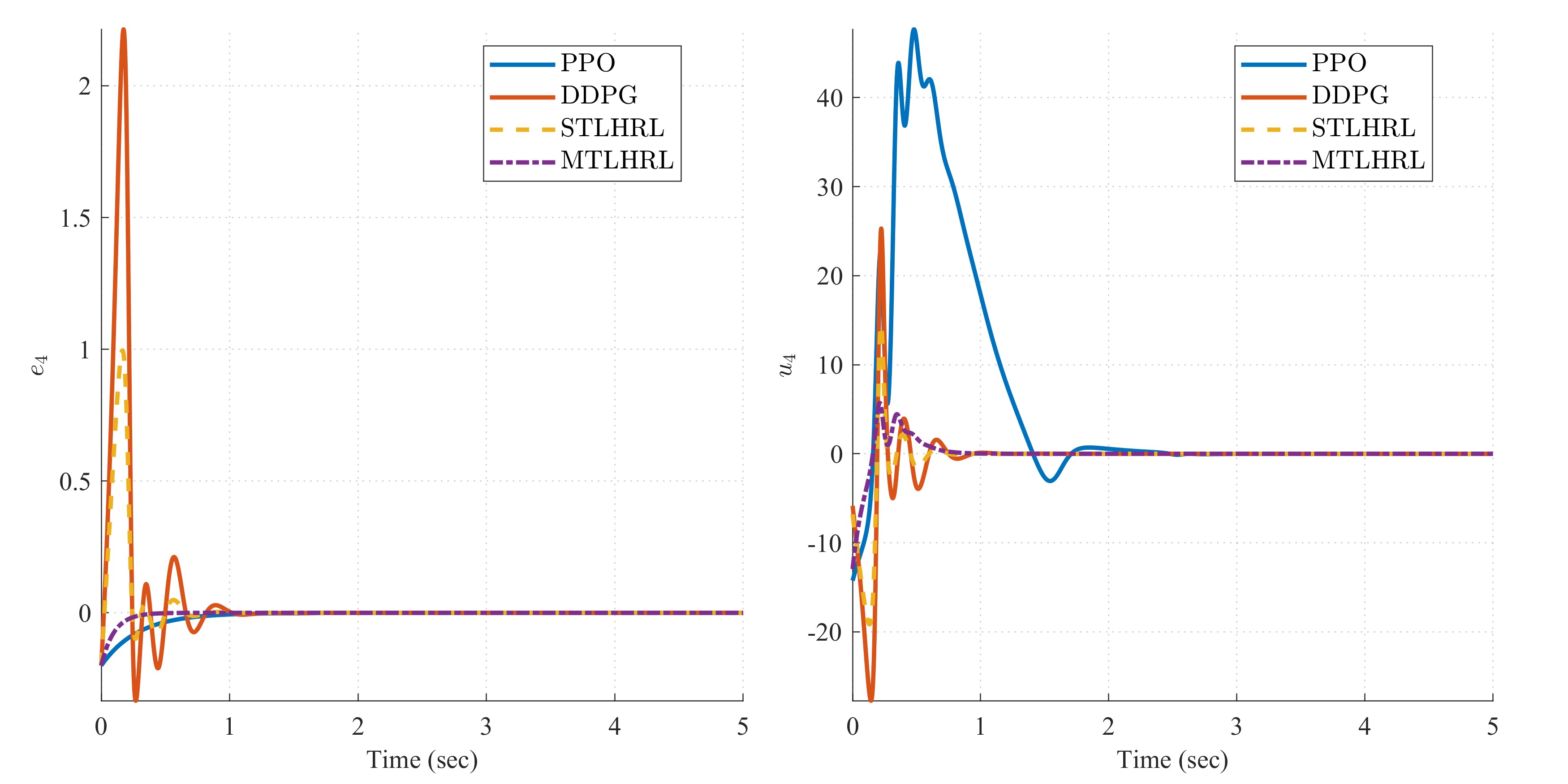} \\
    (c) & (d) \\
    \includegraphics[width=0.45\textwidth]{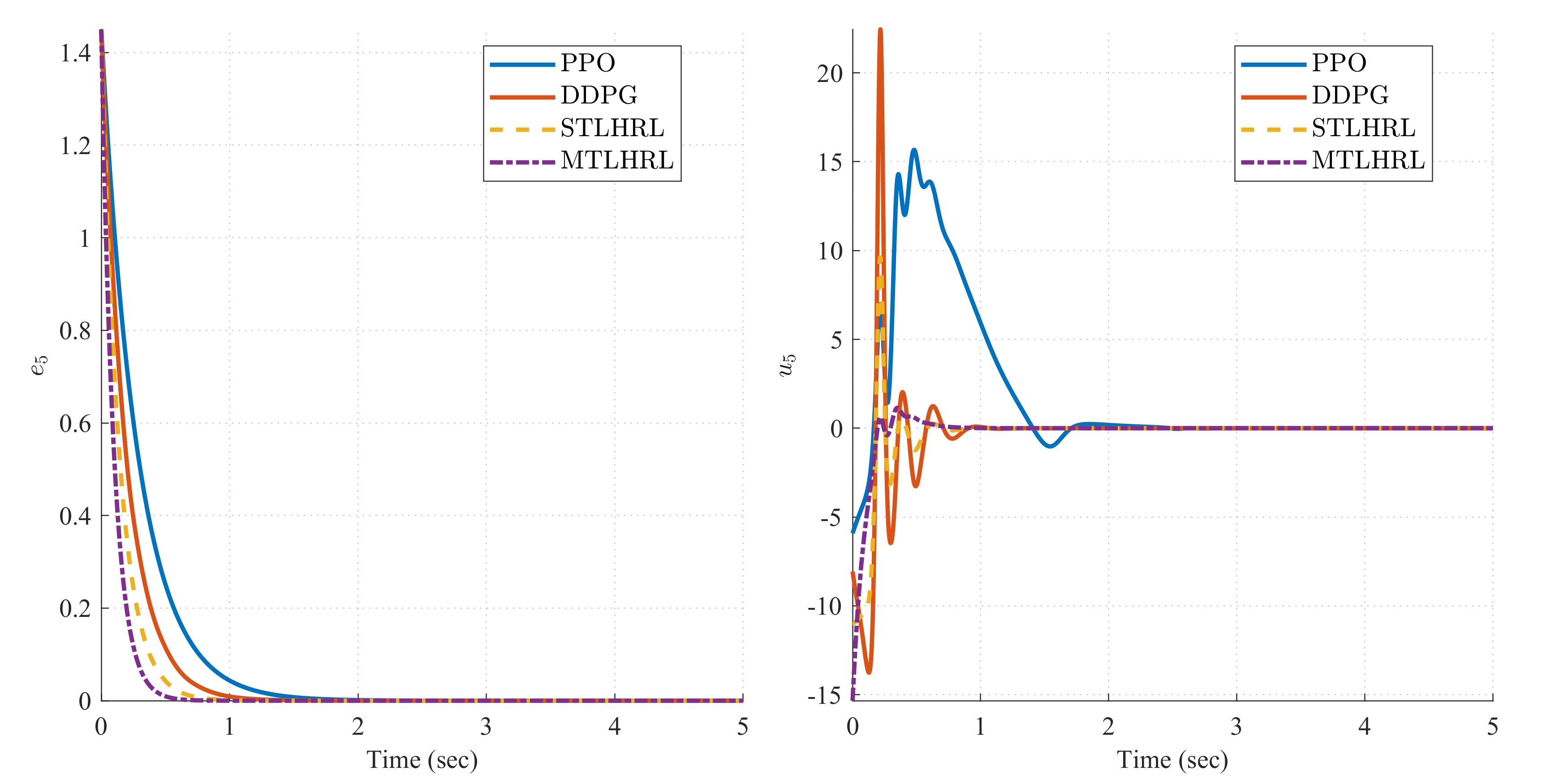} & \includegraphics[width=0.45\textwidth]{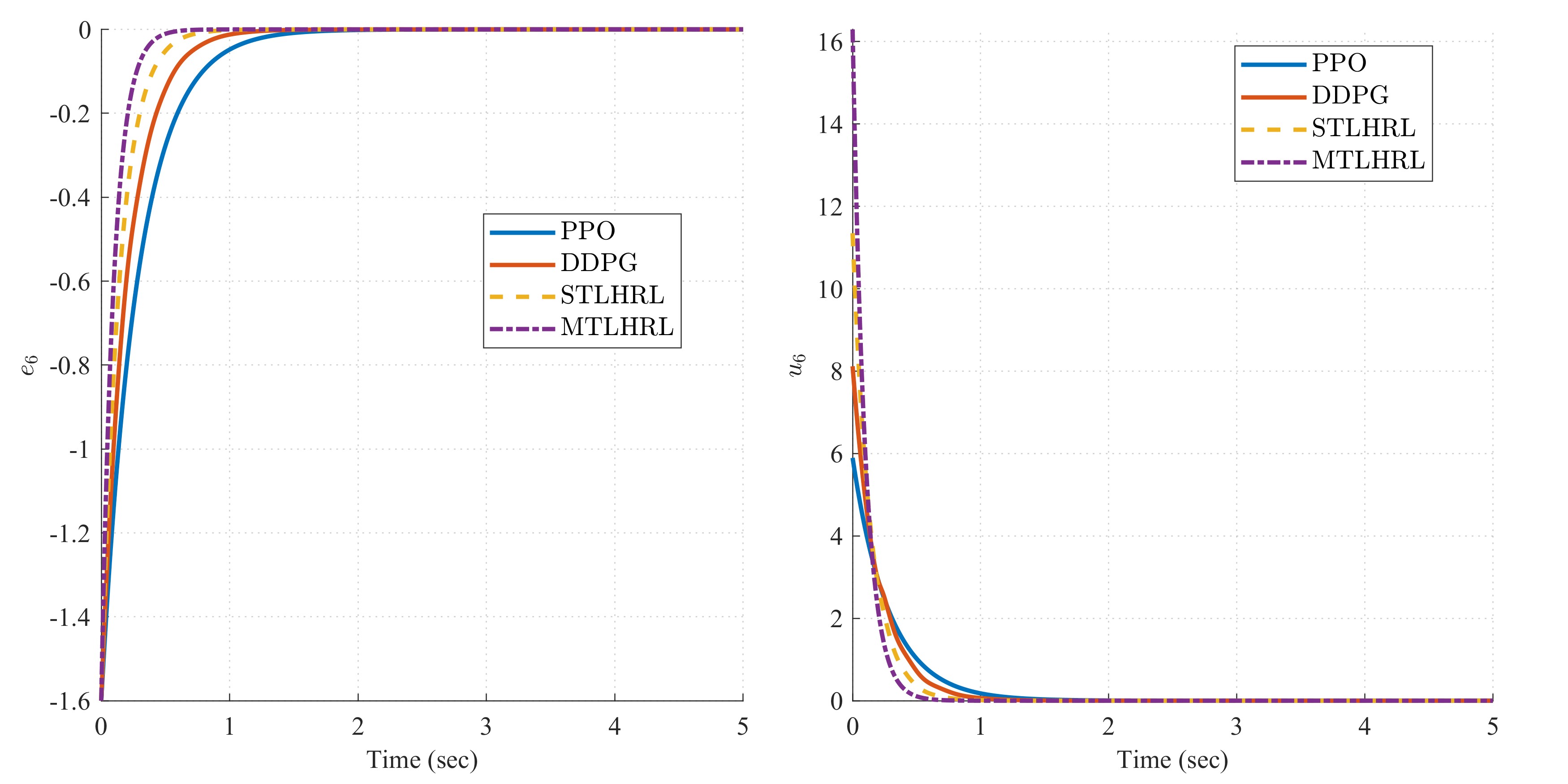} \\
    (e) & (f) \\
    \includegraphics[width=0.45\textwidth]{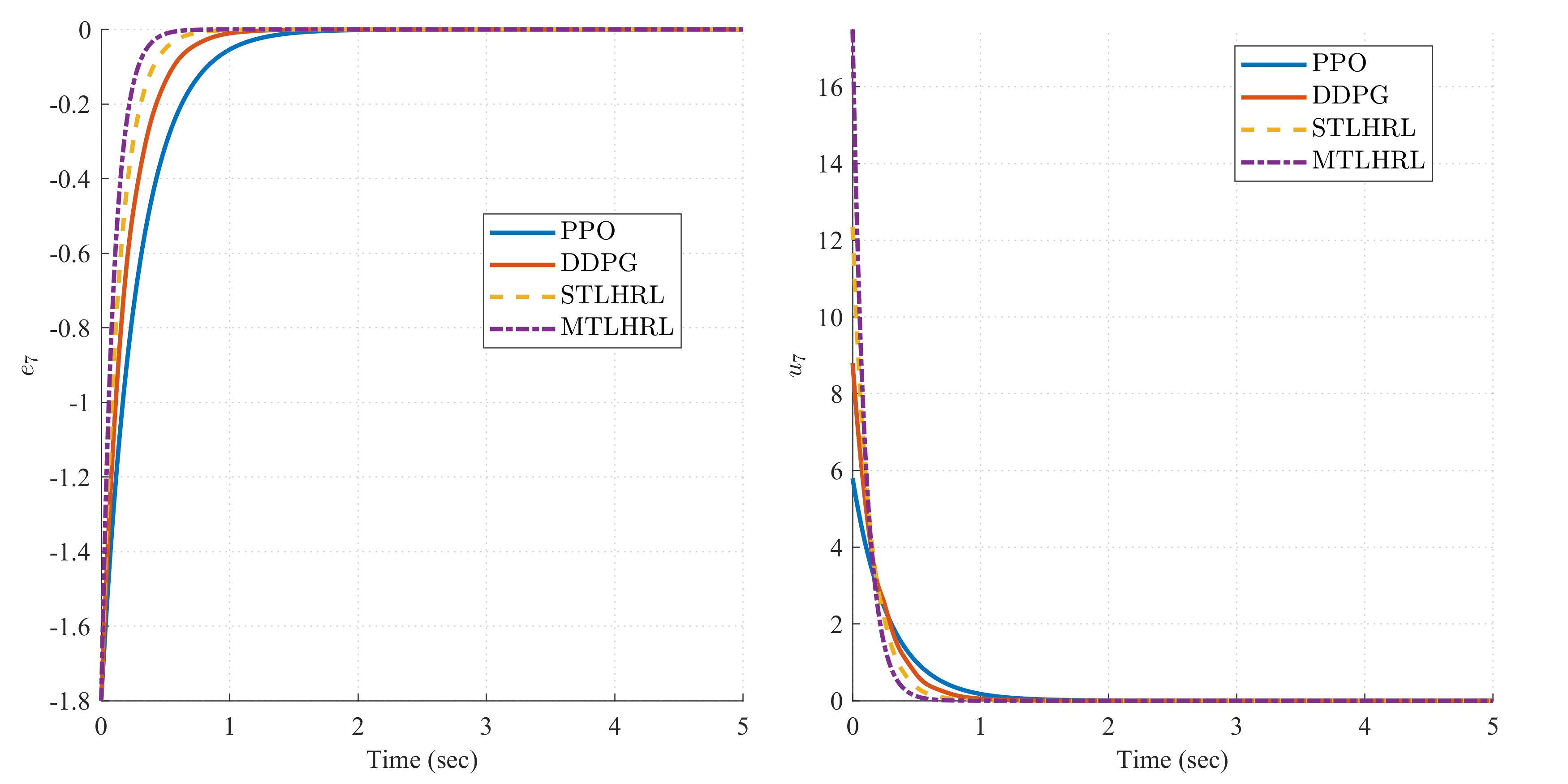} & \includegraphics[width=0.45\textwidth]{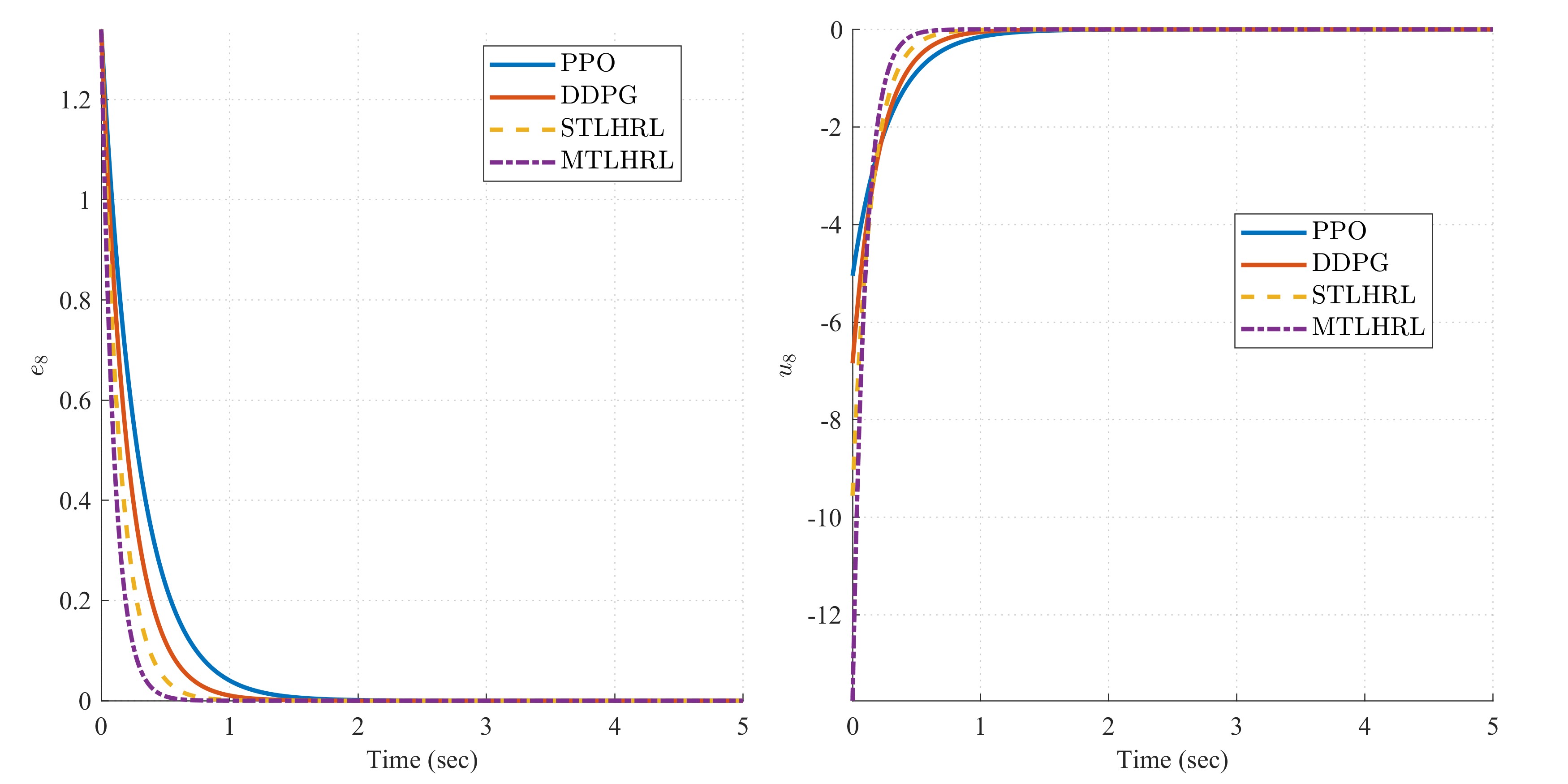} \\
    (g) & (h) \\
\end{tabular}
\caption{Performance comparison among four approaches—PPO, DDPG, STLHRL, MTLHRL—for all eight states.}
\label{fig:8d_states}
\end{figure}

From Figure \ref{fig:8d_states}, MTLHRL delivers superior results, reaching zero synchronization deviation most rapidly while sustaining minimal residual errors. It additionally uses the most conservative actuation efforts, emphasizing its effectiveness.

\begin{figure}[H]
\centering
\includegraphics[width=0.6\textwidth]{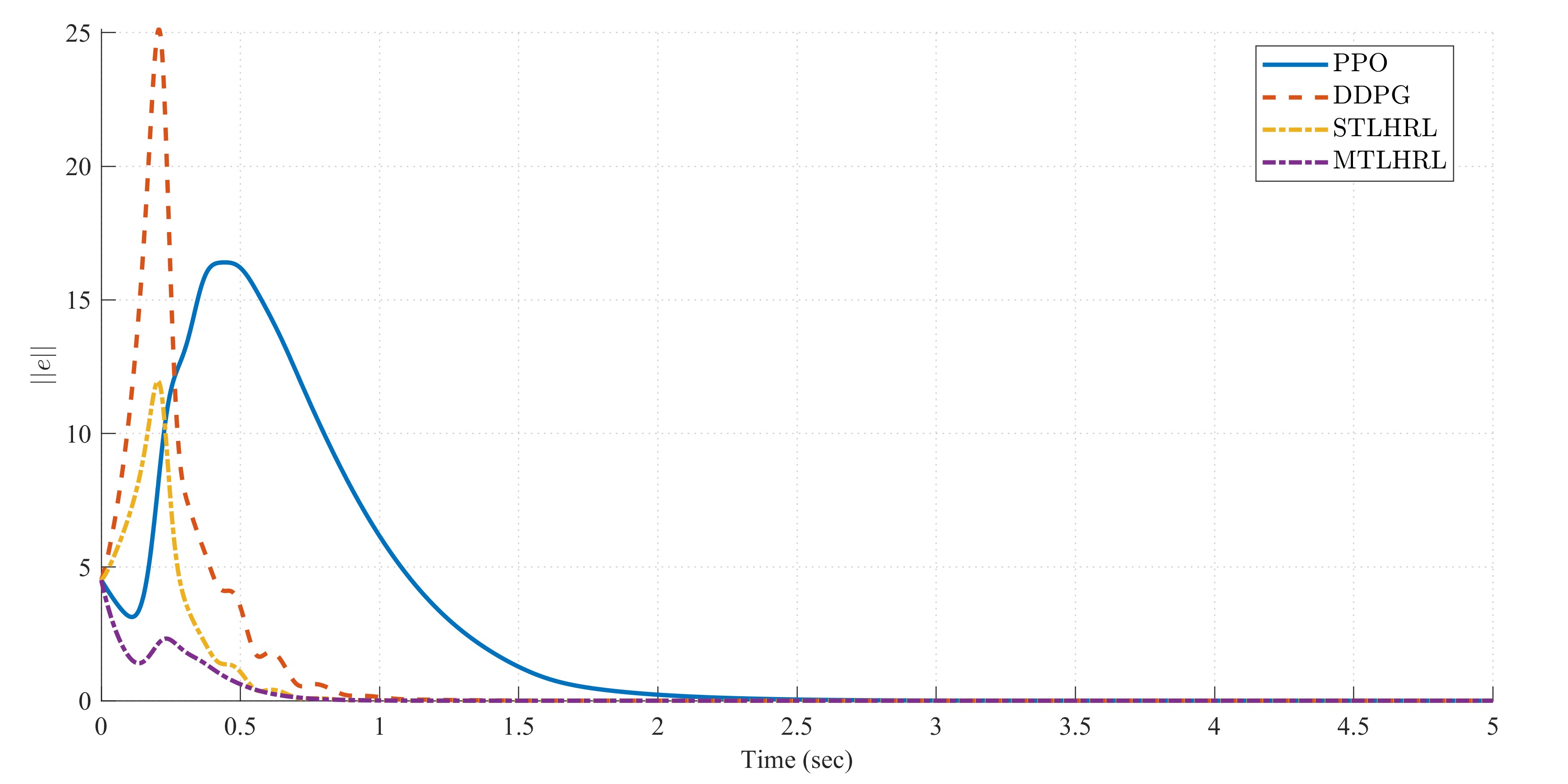}
\caption{Euclidean norms of synchronization errors under PPO, DDPG, STLHRL, and MTLHRL for the eight states.}
\label{fig:8d_euclidean}
\end{figure}

Figure \ref{fig:8d_euclidean} illustrates the evolution of the Euclidean norm of synchronization errors for the 8D hyperchaotic system under four control strategies—PPO, DDPG, STLHRL, and MTLHRL. The results clearly show that MTLHRL achieves the fastest and smoothest convergence toward zero error, with negligible oscillations and minimal overshoot, indicating high stability and robustness. In contrast, PPO exhibits the slowest decay and the largest transient peak, reflecting its limited ability to manage nonlinear coupling and chaotic fluctuations. DDPG improves upon PPO with a shorter settling time but still suffers from pronounced overshoot. Overall, MTLHRL stands out with the smallest starting deviations, quickest settling times, and tiniest long-term inaccuracies. PPO performs the worst, DDPG improves upon PPO, STLHRL positions between DDPG and MTLHRL, yet MTLHRL outperforms all others. For a quantitative evaluation of the proposed method in synchronizing the 8D hyperchaotic system, Table \ref{tab:8d_performance} is presented.

\begin{table}[H]
\centering
\caption{Performance Indices for Synchronization of 8D Hyperchaotic System}
\label{tab:8d_performance}
\begin{tabular}{lcc}
\toprule
Controller & IAE & ISE \\
\midrule
PPO & 12.845 & 18.672 \\
DDPG & 9.321 & 14.105 \\
STLHRL & 6.789 & 9.456 \\
MTLHRL & 3.912 & 5.678 \\
\bottomrule
\end{tabular}
\end{table}

Table \ref{tab:8d_performance} highlights these enhancements numerically, with MTLHRL recording the lowest IAE and ISE metrics, reflecting optimal aggregate error suppression and robust synchronization amid hyperchaotic divergences, process noise, and high-dimensional couplings. It surpasses the comparison methods, where PPO shows the largest errors owing to its non-adaptive nature, and DDPG along with STLHRL provide moderate improvements through RL elements yet lag behind MTLHRL's integration of multi-timescale hierarchies and Lyapunov constraints for ensured stability in stochastic nonlinear environments.

\subsection{5-DOF Robot Manipulator}

The 5-DOF configuration introduces significant nonlinearities and dynamic couplings between joints, making precise control particularly challenging—especially in the presence of external disturbances, parameter uncertainties, sensor noise, and actuator stochasticity. The manipulator’s dynamics are described by the standard Euler–Lagrange formulation in stochastic form to account for noises and disturbances:
$$M(q) \ddot{q} + C(q, \dot{q}) \dot{q} + G(q) = u_t + w_t + \sigma_a u_t \, dW_t^a + d_t,$$
where $q \in \mathbb{R}^5$ is the vector of joint positions, $\dot{q}$ and $\ddot{q}$ are the joint velocities and accelerations, respectively, $M(q) \in \mathbb{R}^{5 \times 5}$ denotes the positive-definite inertia matrix, $C(q, \dot{q}) \in \mathbb{R}^{5 \times 5}$ is the Coriolis and centrifugal matrix, $G(q) \in \mathbb{R}^5$ represents the gravitational torque vector, $u_t \in \mathbb{R}^5$ is the control input vector (subject to actuator stochasticity via $\sigma_a u_t , dW_t^a$), $w_t \sim \mathcal{N}(0, 0.05)$ models sensor-like process noise, and $d_t$ incorporates external disturbances \cite{mirzaee2024type}. Observations are further corrupted by sensor noise $v_t \sim \mathcal{N}(0, 0.01 I)$.

The initial conditions are specified as $x(0) = [-1, -2, 2, 1, 0]^\top$ and $x'(0) = [0.5, 1, -1, -0.5, 0]^\top$. The desired trajectory for each joint is defined as a phase-shifted sinusoidal signal:
$$x_d(t) = A \cdot \begin{bmatrix}
\sin(ft + \pi/5) \\
\sin(ft + 2\pi/5) \\
\sin(ft + 3\pi/5) \\
\sin(ft + 4\pi/5) \\
\sin(ft + 6\pi/5)
\end{bmatrix}, \quad A = 2, \quad f = 0.5.$$
To simulate external disturbances, a time-limited sinusoidal disturbance vector is added between $t = 10 , \text{s}$ and $t = 20 , \text{s}$:
$$\text{Dist}(t) = 0.5 \cdot A \cdot \begin{bmatrix}
\omega(t) \cdot \sin(ft + \pi/5) \\
0.9 \cdot \omega(t) \cdot \sin(ft + 2\pi/5) \\
\omega(t) \cdot \sin(ft + 3\pi/5) \\
0.9 \cdot \omega(t) \cdot \sin(ft + 2\pi/5) \\
\omega(t) \cdot \sin(ft + 2\pi/5)
\end{bmatrix}, \quad A = 6.5, \quad f = 4.0,$$
where $\omega(t) = u(t-10) - u(t-20)$ is a window function that activates the disturbance only during the specified interval. This setup benchmarks the control system’s ability to maintain trajectory tracking accuracy under non-ideal conditions. The detailed derivation and physical parameters of the 5-DOF manipulator can be found in \cite{mirzaee2024type}. Figure \ref{fig:5dof_states} shows the performance of the controllers in terms of joint positions and control siganls.

\begin{figure}[h]
\centering
\includegraphics[width=0.6\textwidth]{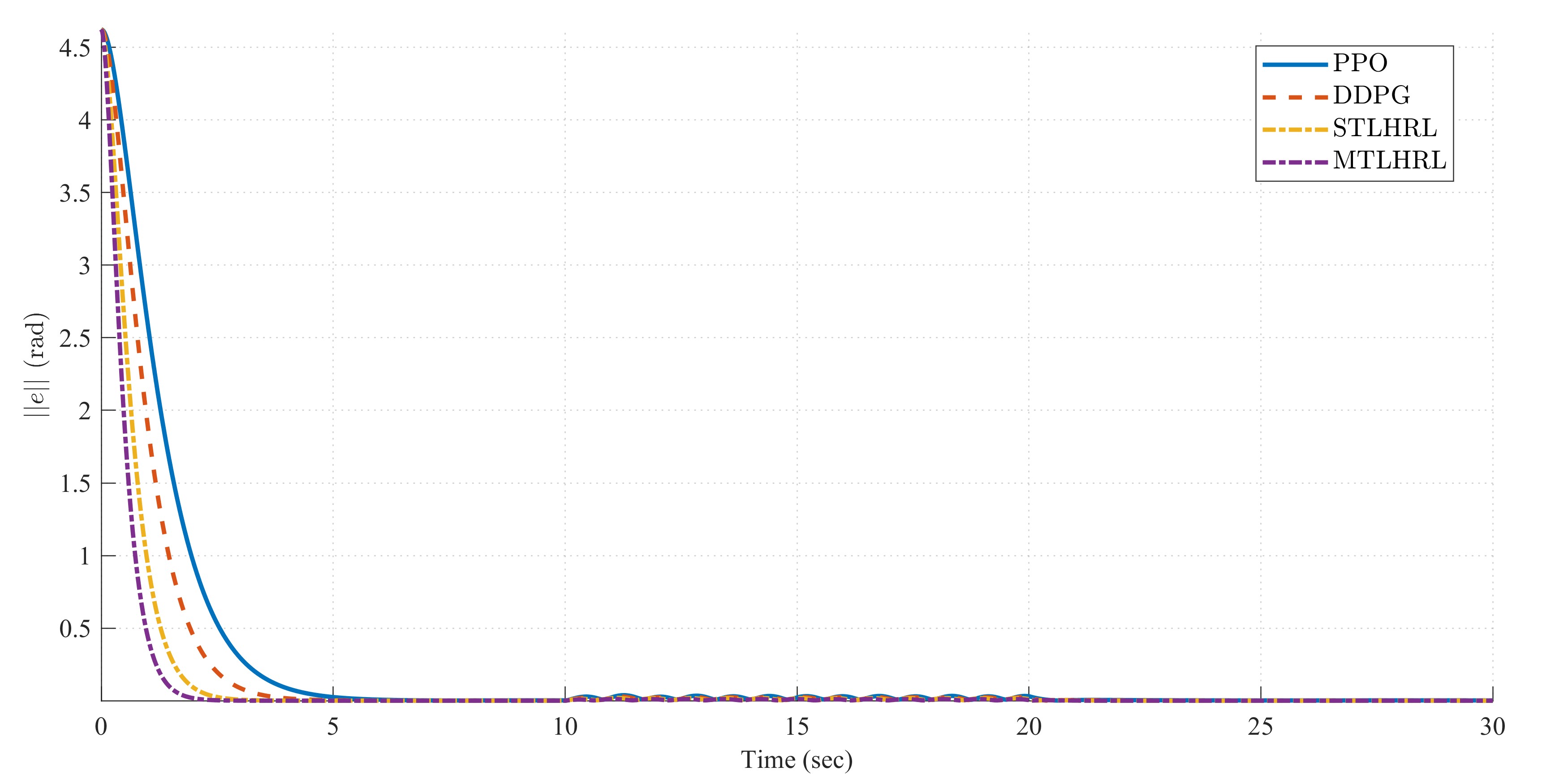}
\caption{Euclidean norm of state errors for four controllers—PPO, DDPG, STLHRL, MTLHRL—across states of the 5-DOF robot manipulator system.}
\label{fig:5dof_euclidean}
\end{figure}

\begin{figure}[h]
\centering
\begin{tabular}{cc}
    \includegraphics[width=0.45\textwidth]{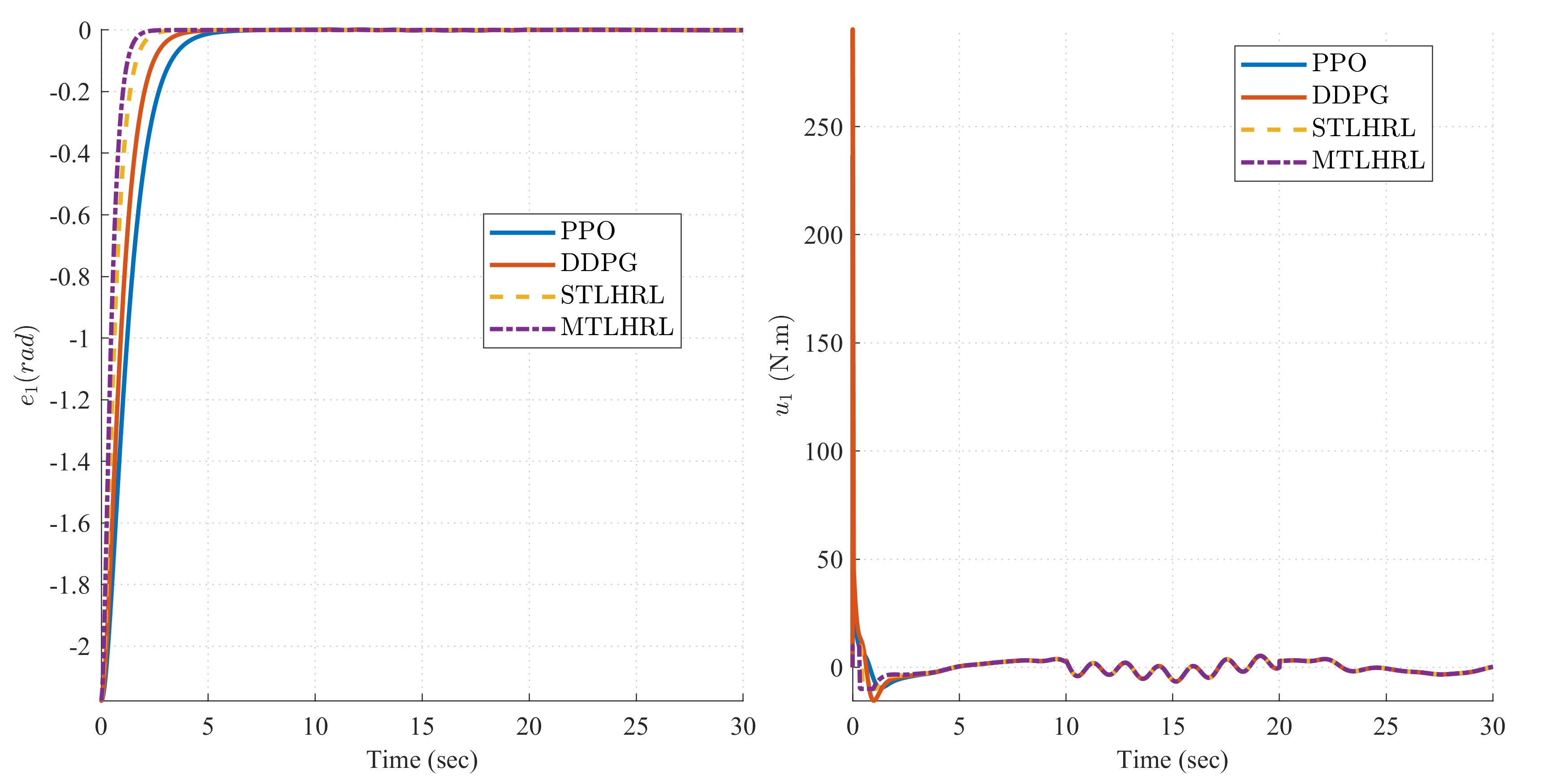} & \includegraphics[width=0.45\textwidth]{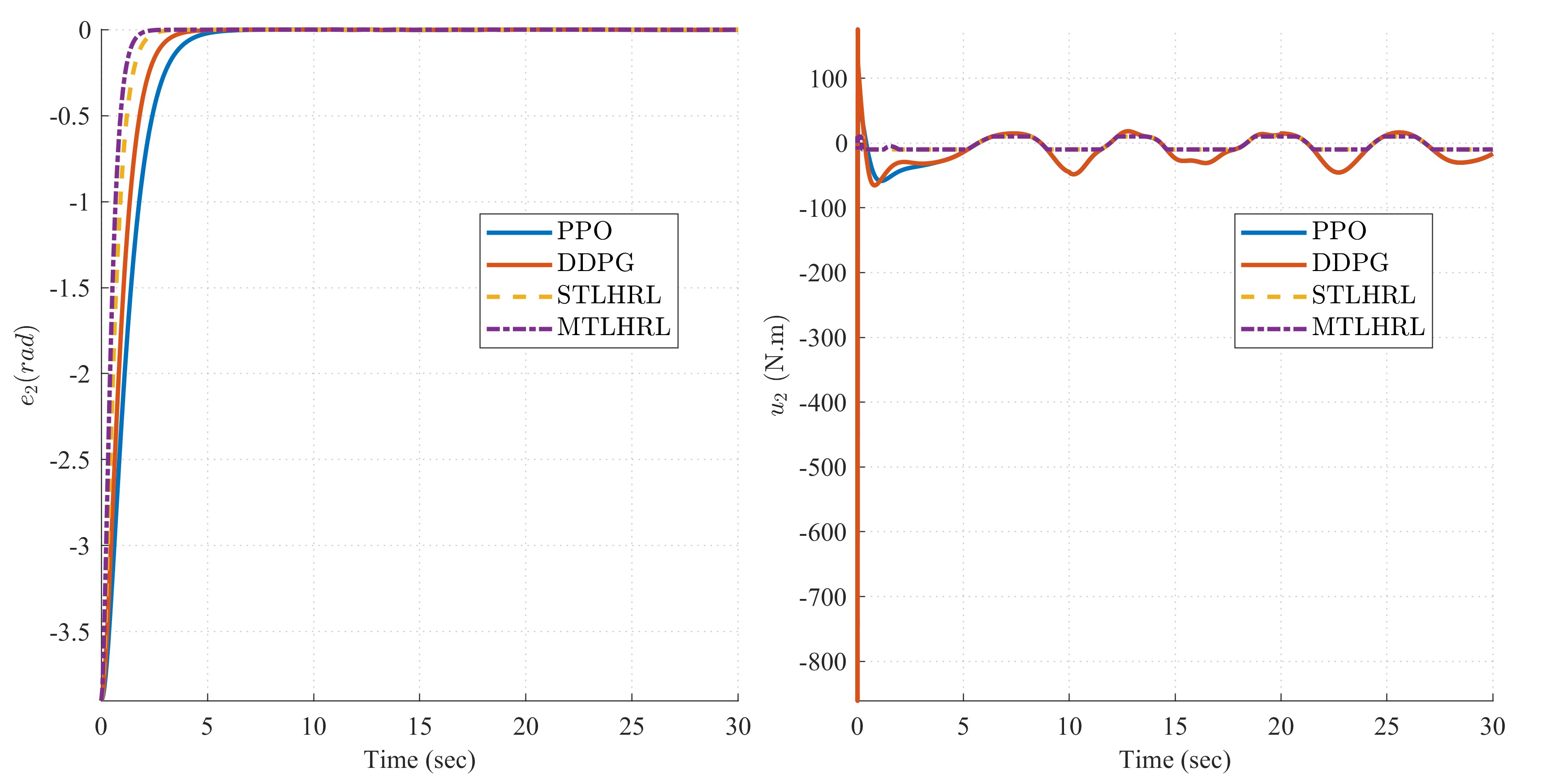} \\
    (a) & (b) \\
    \includegraphics[width=0.45\textwidth]{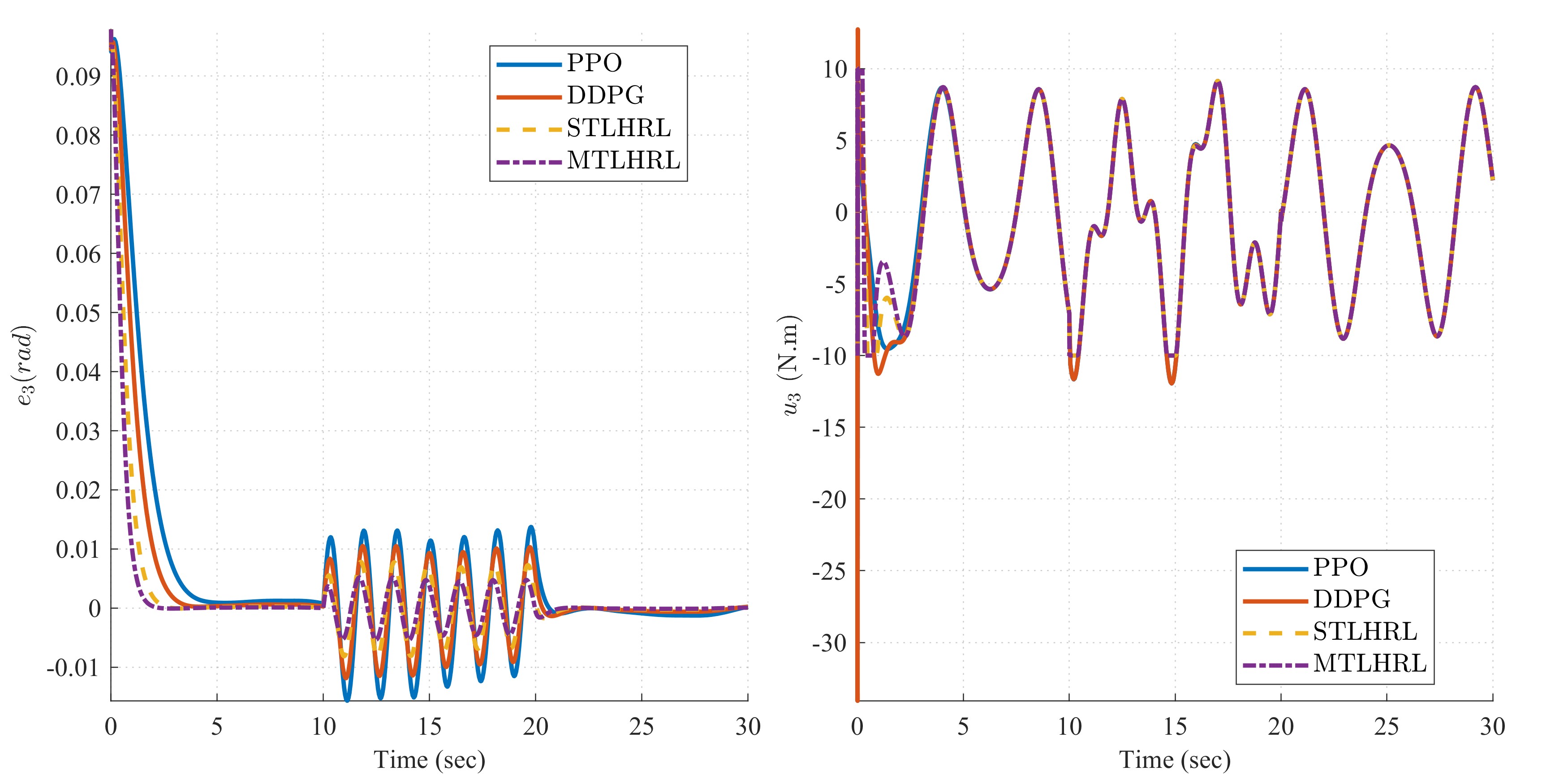} & \includegraphics[width=0.45\textwidth]{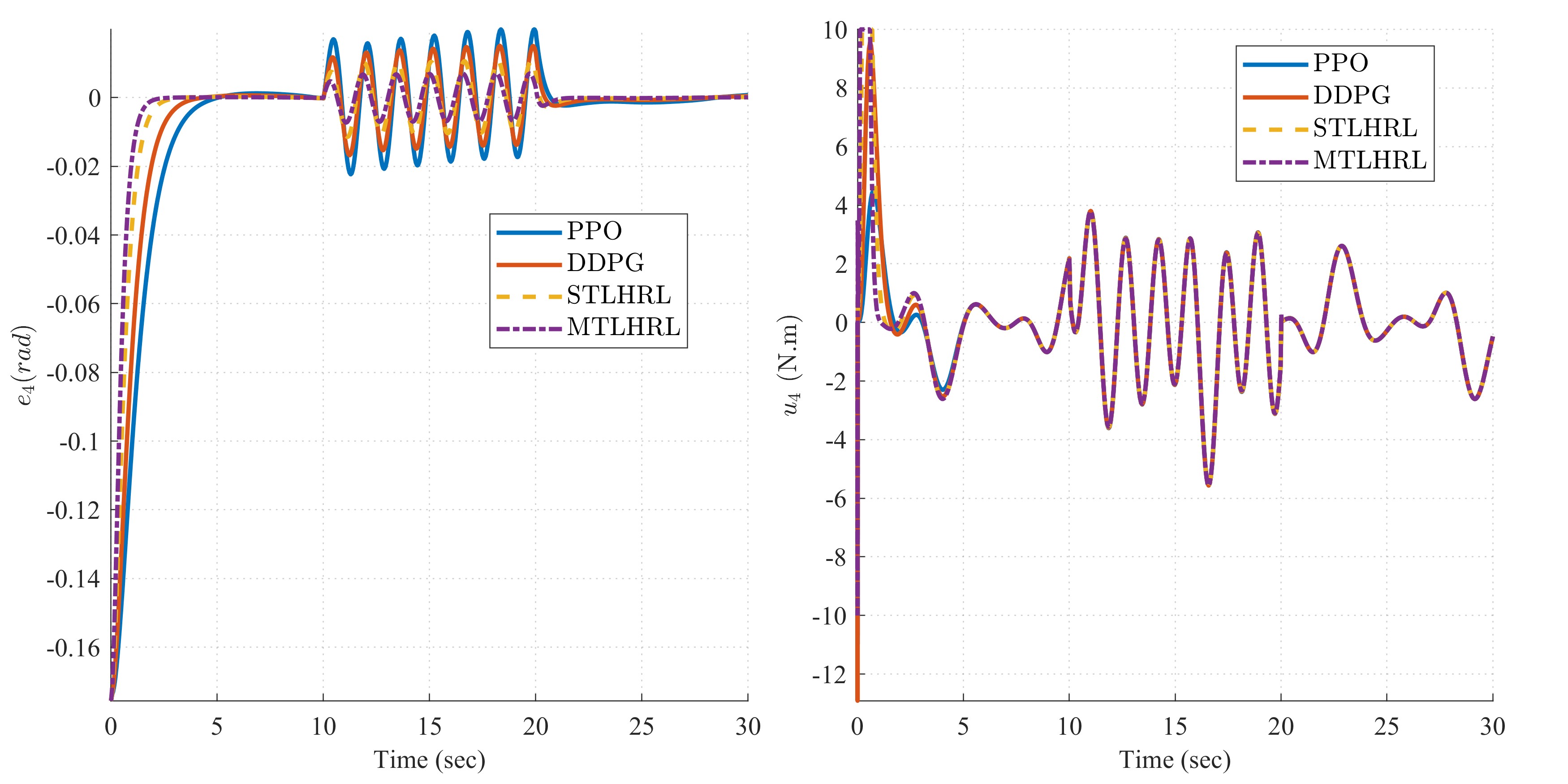} \\
    (c) & (d) \\
    \multicolumn{2}{c}{\includegraphics[width=0.45\textwidth]{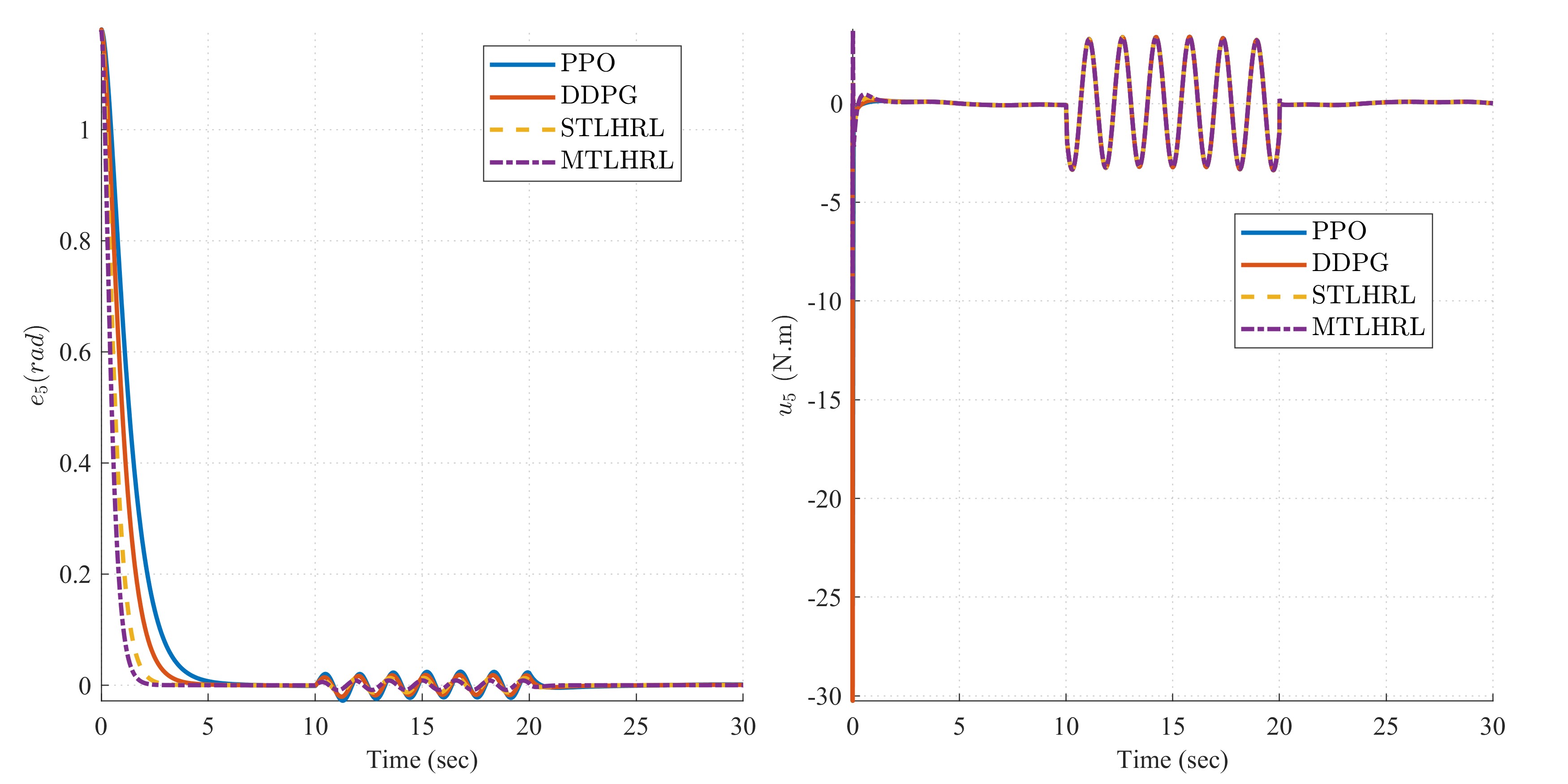}} \\
    \multicolumn{2}{c}{(e)} \\
\end{tabular}
\caption{Comparison of performance for four controllers—PPO, DDPG, STLHRL, MTLHRL—across states of the 5-DOF robot manipulator system.}
\label{fig:5dof_states}
\end{figure}

Figure \ref{fig:5dof_euclidean} and \ref{fig:5dof_states} demonstrate that the proposed MTLHRL controller significantly enhances tracking performance in terms of both error minimization speed and steady-state accuracy compared to PPO, DDPG, and STLHRL baselines. To enable a numerical assessment of the suggested approach in managing the 5-DOF manipulator setup, Table \ref{tab:5dof_performance} is included.

\begin{table}[H]
\centering
\caption{Performance Indices for Controlling of 5-DOF Manipulator System}
\label{tab:5dof_performance}
\begin{tabular}{lcc}
\toprule
Controller & IAE & ISE \\
\midrule
PPO & 4.512 & 7.823 \\
DDPG & 3.678 & 5.912 \\
STLHRL & 2.789 & 4.156 \\
MTLHRL & 1.623 & 2.489 \\
\bottomrule
\end{tabular}
\end{table}

Table \ref{tab:5dof_performance} quantifies these improvements, where MTLHRL achieves the lowest Integral Absolute Error (IAE) and Integral Square Error (ISE) values, indicating superior cumulative error reduction and energy-efficient tracking under disturbances, sensor noise, and actuator stochasticity. This outperforms the baselines, with PPO exhibiting the highest errors due to its lack of adaptive learning, while DDPG and STLHRL show intermediate gains from RL components but fall short of MTLHRL's multi-timescale Lyapunov constraints for stability in high-dimensional stochastic settings.

\section{Conclusion}
\label{sec:conclusion}

In this paper, we introduced the MTLHRL framework, a novel approach designed to tackle the adaptive control of high-dimensional stochastic dynamical systems governed by SDEs, amid challenges like sensor noise, actuator stochasticity, external disturbances, and hyperchaotic behaviors. By integrating hierarchical policy decomposition across multiple timescales with Lyapunov-based stability constraints, MTLHRL ensures robust performance, stochastic stability, and efficient learning in environments where traditional methods struggle with dimensionality curses, instability, and prolonged horizons. Extensive MATLAB simulations on two demanding benchmarks—an 8D hyperchaotic system and a 5-DOF robotic manipulator—validated MTLHRL's superiority over baselines including PPO, DDPG, and STLHRL. In the hyperchaotic scenario, MTLHRL achieved the fastest synchronization error convergence with the lowest cumulative errors (IAE of 3.912 and ISE of 5.678), effective chaos suppression, and minimal control effort despite process noise and extreme sensitivities. For the robotic manipulator, it excelled in trajectory tracking under compounded disturbances and noises, yielding the best metrics (IAE of 1.623 and ISE of 2.489), faster transient responses, and superior steady-state precision, as demonstrated through state trajectories and Euclidean error norms. These results underscore MTLHRL's scalability, versatility across chaotic and mechanical domains, and theoretical grounding in stochastic Lyapunov theory, offering mean-square boundedness and disturbance rejection without excessive computational overhead. Future work may extend MTLHRL to real-time hardware implementations (e.g., physical robots), partial observability via POMDPs, or multi-agent systems, further bridging RL with control theory for autonomous applications in robotics, aerospace, and beyond. Overall, MTLHRL represents a significant advancement in stable, adaptive control for complex real-world dynamics.

\bibliographystyle{apacite}
\bibliography{references}

\end{document}